\documentclass[10pt,twocolumn,letterpaper]{article}

\usepackage[pagenumbers]{cvpr} 

\usepackage{amsmath,amsfonts,bm}

\def\eqref#1{equation~\ref{#1}}

\def\1{\bm{1}}

\DeclareMathAlphabet{\mathsfit}{\encodingdefault}{\sfdefault}{m}{sl}
\SetMathAlphabet{\mathsfit}{bold}{\encodingdefault}{\sfdefault}{bx}{n}

\DeclareMathOperator*{\argmin}{arg\,min}

\definecolor{cvprblue}{rgb}{0.21,0.49,0.74}
\usepackage[pagebackref,breaklinks,colorlinks,allcolors=cvprblue]{hyperref}

\usepackage[font=footnotesize,skip=0pt]{caption}

\usepackage{hyperref}       
\usepackage{url}            
\usepackage{booktabs}       
\usepackage{amsfonts}       
\usepackage{nicefrac}       
\usepackage{microtype}      
\usepackage{xcolor}         

\usepackage{adjustbox}
\definecolor{ForestGreen}{RGB}{34,139,34}
\definecolor{Crimson}{RGB}{220,20,60}

\usepackage{epsfig}
\usepackage{tabularx}
\usepackage{graphicx}
\usepackage{amsmath}
\usepackage{amsthm}
\usepackage{amssymb}
\usepackage{multirow}
\usepackage{booktabs}
\usepackage{subcaption}
\usepackage{algorithm}
\usepackage{colortbl}
\usepackage{tcolorbox}
\usepackage{algpseudocode}
\usepackage{wrapfig} 

 \newtheorem{proposition}{Proposition}

\usepackage{xspace}
\def\ours{\texttt{V-Reason}\@\xspace}
\def\ourslite{\texttt{V-Reason(Lite)}\@\xspace}

\def\paperID{15870} 
\def\confName{CVPR}
\def\confYear{2026}

\title{Video Reasoning Without Training}

\author{
Deepak Sridhar$^{1,2}$\thanks{Equal contribution. Work done when Deepak Sridhar was an intern at Qualcomm AI Research. $^\dag$Qualcomm AI Research is an initiative of Qualcomm Technologies, Inc.} \quad
Kartikeya Bhardwaj$^{1}$\footnotemark[1] \quad
Jeya Pradha Jeyaraj$^{1}$ \quad
Nuno Vasconcelos$^{2}$\\
Ankita Nayak$^{1}$ \quad
Harris Teague$^{1}$\\
$^{1}$Qualcomm AI Research$^\dag$\quad
$^{2}$University of California, San Diego\\
\small{\texttt{desridha@ucsd.edu, kbhardwa@qti.qualcomm.com}}
}

\begin{document}
\maketitle

\begin{abstract}
Video reasoning using Large Multimodal Models (LMMs) relies on costly reinforcement learning (RL) and verbose chain-of-thought, resulting in substantial computational overhead during both training and inference. Moreover, the mechanisms that control the thinking process in these reasoning models are very limited. In this paper, we use the entropy of the model’s output distribution as a signal to study and guide reasoning behavior. We discover that high-quality models exhibit a characteristic pattern of \textit{micro-exploration} and \textit{micro-exploitation} cycles, followed by a later entropy peak (i.e., longer thinking) and a lower final entropy, indicating more deliberate exploration and confident convergence (i.e., avoid excessive randomness while the model is exploring or thinking through an answer). We then use these novel, theoretically-grounded insights to introduce \ours{} (\underline{V}ideo-\underline{Reason}), an inference-time optimization method that adapts the value cache of the LMM through a lightweight, trainable controller. Our proposed controller is guided by an entropy-based objective, to tune the model's behavior directly at inference, without using any RL or supervised fine-tuning. Our experiments show that \ours{} significantly outperforms the base instruction-tuned models on many video reasoning datasets, narrowing the gap with RL models to within \textbf{0.6\%}  accuracy on average. We achieve this without any training, while offering efficiency benefits: \ours{} uses \textbf{58.6\%} fewer tokens than the RL model.\vspace{-4mm}

\end{abstract}

\section{Introduction}\label{sec:intro}
Reasoning with generative AI models, such as Large Language or Large Multimodal Models (LLMs/LMMs), has gained substantial attention recently. This capability is implemented by asking the model to ``think'' about a problem, before making a final recommendation, and can be accomplished by several approaches, including Chain-of-Thought (CoT)~\citep{wei2022cot}, supervised fine-tuning with CoT (CoT-SFT)~\citep{liu2025acereasonnemotron11advancingmath,feng2025videor1}, or reinforcement learning (RL) with a \textit{thinking-before-answering} format~\citep{guo2025deepseekr1, openai2024openaio1card}. Although initial progress was shown mainly for LLMs, such ideas have now been extended to video reasoning problems~\citep{feng2025videor1,li2025videochatr1,zhang2025videott,cheng2025vstar,wang2024sokbench} by exploiting LMMs. Although successful, CoT-SFT, and RL-based methods tend to be highly computationally intensive, both for training and inference, due to the long thinking traces that they tend to produce. These costs are particularly exacerbated for video, due to the high resolution and multiple frames involved in the reasoning process. 

Recently, RL-based reasoning has been viewed as a sampling process~\citep{song2025outcome, zhao2025learning} to more effectively \textit{search} for reasoning traces from \textit{pretrained knowledge} of the baseline model. 
This was also recently noted in~\cite{yue2025does} where the authors study whether RL-based models really expand reasoning abilities beyond the base model.  
We thus ask whether this search for reasoning traces can be performed in a \textit{training-free} way. 

Specifically, we consider the following \textbf{key questions}: 
\begin{enumerate}
    \item Can inference-time metrics characterize the thinking process of video reasoning models? If yes, can these metrics differentiate between higher- and lower-quality LMMs?
    \item Can such metrics be used to formulate novel inference-time optimization objectives for video reasoning without requiring additional model training?
    \vspace{-1mm}
\end{enumerate}

\begin{figure*}[t] 
  \centering
   \includegraphics[width=1.0\linewidth]{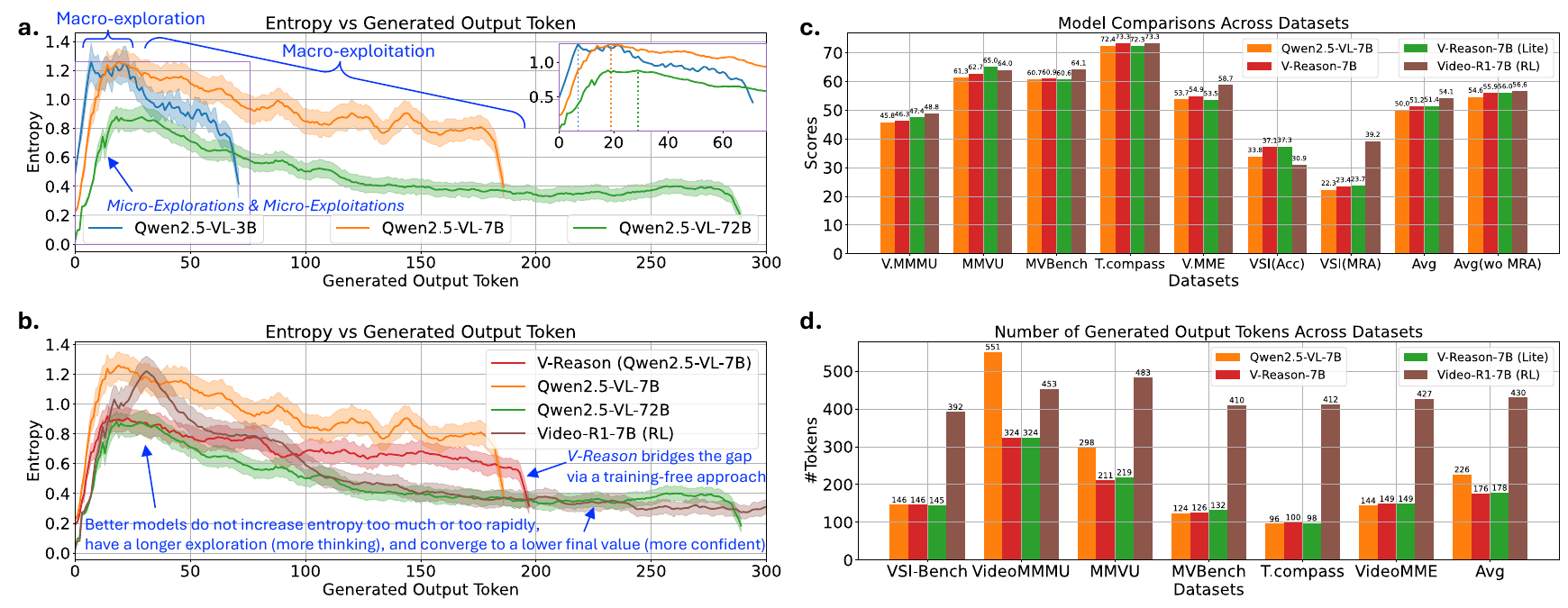}
\caption{\ours{} Overview: (a)~Entropy of the output distribution averaged over the MMVU~\citep{zhao2025mmvu} dataset of 625 videos. We see clear macro-exploration and macro-exploitation phases with bigger, more accurate models showing lower overall entropy (lower and later peak, followed by a lower final entropy during the macro-exploitation). We use these key insights to adapt a model's behavior in a training-free way using an inference-time optimization technique. (b)~Applying \ours{} on Qwen2.5-VL-7B-Instruct makes its entropy behave more similarly to the larger or the RL-trained Video-R1-7B model. 
(c)~\ours{} bridges the accuracy gap with Video-R1-7B to within 0.6\% for tasks well-represented in the pretrained model's knowledge (Avg. without MRA). 
(d)~\ours{}  also significantly reduces the total output tokens compared to all models due to a dedicated entropy minimization phase.}
   \label{fig:entropyIntro}
\end{figure*}

To answer these questions, we first analyze the model's output distribution entropy at generation step $t$ computed as $H_t=-\sum_{i\in \mathcal{V}} p_t^i \log p_t^i$ ($\mathcal{V}$ refers to vocabulary of the model) for instruction-tuned LMMs of various sizes, as shown in Fig.~\ref{fig:entropyIntro}(a). This analysis reveals two broad trends: (\textit{i})~all models exhibit a pattern of increasing and then decreasing entropy as tokens are generated, and (\textit{ii})~larger, more accurate models exhibit both a \textit{later entropy peak and a lower final entropy}, indicating more confident reasoning (see Fig.~\ref{fig:entropyIntro}(a) and its inset for Qwen-VL-Instruct models).

The first trend above can suggest a formal definition of the ``thinking'' in terms of output distribution entropy. As the model starts generating a response, it seems to be uncertain and \textit{searches} through multiple solution trajectories, which can explain the increase in its output entropy. We denote this as the \textit{macro-exploration} phase. As the generation progresses, the model seems to start identifying the correct thinking thread, and becomes increasingly certain about a solution, resulting in the gradual reduction in the entropy of its output. We denote this as the \textit{macro-exploitation} phase. 

The second trend seems to suggest that entropy should \textit{not} increase too rapidly during the macro-exploration phase. In fact, all models go through a series of \textit{micro-exploration} and \textit{micro-exploitation} cycles (characterized by small increases and decreases of entropy) during both macro phases of the thinking process; see Fig.~\ref{fig:entropyIntro}(a) shaded regions. A delayed entropy peak can suggest that better reasoning models explore more alternative answers, leading to higher accuracy~\citep{wei2022cot, guo2025deepseekr1, openai2024openaio1card}. In this context, more and/or longer cycles of micro-exploration and micro-exploitation can lead to ``longer thinking,'' with lower and delayed entropy peaks and lower final entropy. 
Fig.~\ref{fig:entropyIntro}(b, brown line) shows that the above two observations also hold for an RL-trained Video-R1-7B model~\citep{feng2025videor1}.  This model has a slightly lower and much later entropy peak than the Qwen2.5-VL-7B-Instruct baseline model, which was used to train Video-R1-7B, and the final entropy is very close to that of the significantly larger Qwen2.5-VL-72B-Instruct model. In contrast, smaller models (e.g., 3B) peak early (i.e., think less) and converge prematurely to lower entropy leading to confident but incorrect answers (see Fig.~\ref{fig:entropyIntro}(a)). It suggests that a shorter thinking phase can lead to lower reasoning accuracy. 

Building on these observations, we propose \ours{}, a \textit{training-free} inference-time optimization method that directly modulates the micro-exploration and micro-exploitation behavior of the baseline instruction-tuned models to enhance their thinking capabilities. \ours{} introduces a small, trainable controller that adapts the LMM value cache using an entropy-based objective, requiring no supervision, dataset, or RL signal. The objective encourages more pronounced cycles of micro-exploration and micro-exploitation, by inducing the model to more strongly increase/decrease entropy during these cycles, followed by a final entropy minimization phase. This process prevents entropy from rising too fast during macro-exploration and enables the model to achieve a lower final entropy during macro-exploitation, thus making the baseline model behave more like a stronger reasoning model (see Fig.~\ref{fig:entropyIntro}(a,b)). In effect, it encourages the model to think longer but more efficiently by exploring alternatives without producing unnecessary tokens. This mirrors findings from prior studies~\citep{xu2025chain,wu2025when,wang2025wait} showing that shorter yet well-targeted reasoning chains can equal or outperform longer, verbose ones. To enhance efficiency, we further introduce a “lite” variant, \ourslite{}, which reduces memory and computational overhead by evicting 50\% of the lowest-norm video tokens from the KV-cache.

Our results demonstrate that \ours{} and \ourslite{} bridge the gap between baseline instruction tuned models and RL-trained models in terms of accuracy (see Fig.~\ref{fig:entropyIntro}(c)). We empirically observe that this approach is most effective when the solution lies within the pretrained model's knowledge space (e.g., classification tasks we study) but requires better search strategies to surface it. For tasks where this knowledge underrepresented in pretraining (e.g., regression tasks like VSI-MRA), training-based approaches remain more effective.
Moreover, our dedicated entropy minimization phase enables the final solution trajectory to converge faster, thus producing considerably fewer output tokens on average compared to the RL models (see Fig.~\ref{fig:entropyIntro}(d)) which also helps the inference times. Thus, \ours{} and \ourslite{} bridge the gap with the RL-trained model while producing significantly fewer output tokens. 
In summary, the paper makes the following \textbf{key contributions}:
\begin{enumerate}
    \item To our knowledge, the problem of \textit{ video reasoning without training} has not been previously addressed in the literature. We are the first to introduce a training-free, purely inference-time optimization method for video reasoning without SFT or RL.
    \item We propose \ours{} that enhances the micro-exploration and micro-exploitation cycles of the baseline instruction-tuned models to achieve better accuracy. We also provide simple theoretical results for our method.
    \item We show that \ours{} with our proposed entropy based objective induces a lower and delayed entropy peak during macro-exploration and a lower final entropy during macro-exploitation, similar to the patterns observed for the reasoning models trained by RL (see Fig.~\ref{fig:entropyIntro}(b)). 
    \item Extensive experiments on six video reasoning benchmarks show that \ours{} achieves an average improvement of \textbf{1.4\%} over the base model, narrowing the gap to within \textbf{0.6\%} of the RL-trained Video-R1-7B model (see Fig.~\ref{fig:entropyIntro}(c)). We further show gains across model sizes ranging from 3B to 32B and even up to 72B LMMs. We also demonstrate that \ours{} is complementary to multiple SOTA decoding methods. 
    \item Our approach can lead to more efficient reasoning by significantly reducing the total number of output tokens generated (see Fig.~\ref{fig:entropyIntro}(d)). \ours{} produces \textbf{21.4\%} fewer tokens on average than the base Qwen2.5-7B-Instruct model, and \textbf{58.6\%} fewer tokens on average than the RL-trained Video-R1-7B model. This means that its wall-clock inference time is competitive to the base model and up to \textbf{37\%} lower than Video-R1-7B on average.
\end{enumerate}

\section{Related Work}\label{sec:rel}
\noindent\textbf{Reasoning in Large Language Models.}
Reasoning in LLMs can be achieved by chain-of-thought prompting, instruction-tuning with CoTs, or reward-based fine-tuning with RL. Existing work on prompting primarily relies on eliciting better CoT reasoning paths from the model~\citep{kojima2022zeroshotcot,yasunaga2023analogicalreasoners,zhou2023leastmost}. While these methods have achieved high accuracies, few-shot prompting techniques are task-specific, less generalizable and require manual prompt designs for each task. Better prompting techniques require extensive prompt engineering and result in inconsistent performances~\citep{zhou2023humanlevelpromptengineers}. Overall, prompting techniques are limited by model-specific and task-specific tuning~\citep{yang2024llmasoptimizers} making them less favorable. Recent works endeavor to improve the CoT prompting by verification~\citep{golovneva2023pathfinder} that verifies and controls the intermediate steps generated by the model. Such methods still require CoT prompting and are computationally intensive due to the additional verification steps involved. 

Instruction-tuning and reward-based fine-tuning are alternative ways to elicit reasoning in LLMs when additional compute is available for supervision~\citep{magister2023teachingsmallreason, huang2023llmsselfimprove,chung2022scaling}. However, these techniques require supervised CoT data and expensive RL stages to make the model compliant to produce the reasoning or thinking process in specified formats for easy extraction of the answers. Different from the above methods, we seek an efficient framework to elicit reasoning in LMMs at inference without any supervised data or training.

\noindent\textbf{Video Reasoning.}
Video Reasoning methods have been introduced recently~\citep{feng2025videor1,chen2025scalingRLLongVideos} inspired by the success of LLM reasoning. Video-R1~\citep{feng2025videor1} introduces a temporal GRPO loss to specifically improve temporal reasoning capabilities along with a new dataset for training. VideoChat-R1~\citep{li2025videochatr1} introduces a chat model with spatio-temporal reasoning abilities by training with GRPO and rule-based rewards. TinyLLaVA~\citep{zhang2025tinyllava} shows that reasoning can be effective even for smaller models, using a Qwen-3B-VL model trained with standard GRPO and RL-based reward losses. All of the above methods rely on expensive training to elicit reasoning in LMMs for videos; for instance, training TinyLLaVA on 50K samples takes $\sim$3 days on 4 A100 GPUs, and the cost scales prohibitively for larger models (7B, 32B). To overcome this, we propose an efficient framework that leverages inference-time optimization to elicit the reasoning traces from pretrained LMMs, achieving higher accuracy with fewer output tokens compared to RL-trained models.

\noindent\textbf{Inference-time Reasoning Methods.}
Inference-time optimization methods~\citep{chefer2023attend,rout2025rbmodulation} have gained popularity in diffusion models for improving control and consistency. Recent works have explored eliciting reasoning capabilities from LLMs at inference time~\citep{wang2024cotnoprompt,fu2025deepconf}, aiming to reduce computational cost and improve interpretability. 
Decoding strategies such as CoT-Decoding \citep{wang2024cotnoprompt} modifies token selection to surface latent reasoning traces, while 
ThinkLogit \citep{zhang2025logit} manipulates logits with guidance from a smaller preference model to induce longer reasoning chains.
In parallel, sampling-based methods such as min-p \citep{Nguyen2024MinPSampling} and the concurrent approach top-h  \citep{Potraghloo2025TopHDecoding} 
restrict candidate tokens based on probability thresholds or rank cutoffs, improving fluency but without explicitly targeting reasoning. 
Our method is orthogonal to these approaches: rather than filtering outputs, we optimize the model’s intrinsic token distributions during inference and show consistent improvements even when combined with min-p and top-h sampling-based methods.

Other line of works utilize steering to modify the model's behavior for reasoning tasks~\citep{azizi2025asc,belitsky2025cache}. ASC~\citep{azizi2025asc} modifies the hidden states of the model to compress CoT traces by relying on a reasoning-trained model to distinguish concise from verbose reasoning. KV-Cache Steering~\citep{belitsky2025cache} presents a one-shot intervention in the key-value cache to induce reasoning in small LLMs with steering vectors derived from GPT-4o~\citep{openai2024gpt4o}.
In contrast to these works that have \textit{indirect reliance on a reasoning-trained} model, we propose an inference-optimization technique that modulates the value-cache to elicit reasoning using the model’s own entropy as intrinsic feedback\textit{ without any reliance on external model or data}.

\section{Proposed Approach: \ours{}}\label{sec:approach}
In this section, we describe the proposed \ours{}, its inference-time optimization objectives, and supporting theoretical results.
We then address practical aspects, including redundancy reduction in video tokens to lower memory costs, and introduce \ourslite{} for improved efficiency.

\begin{figure*}[t]
  \centering
   \includegraphics[width=0.8\linewidth]{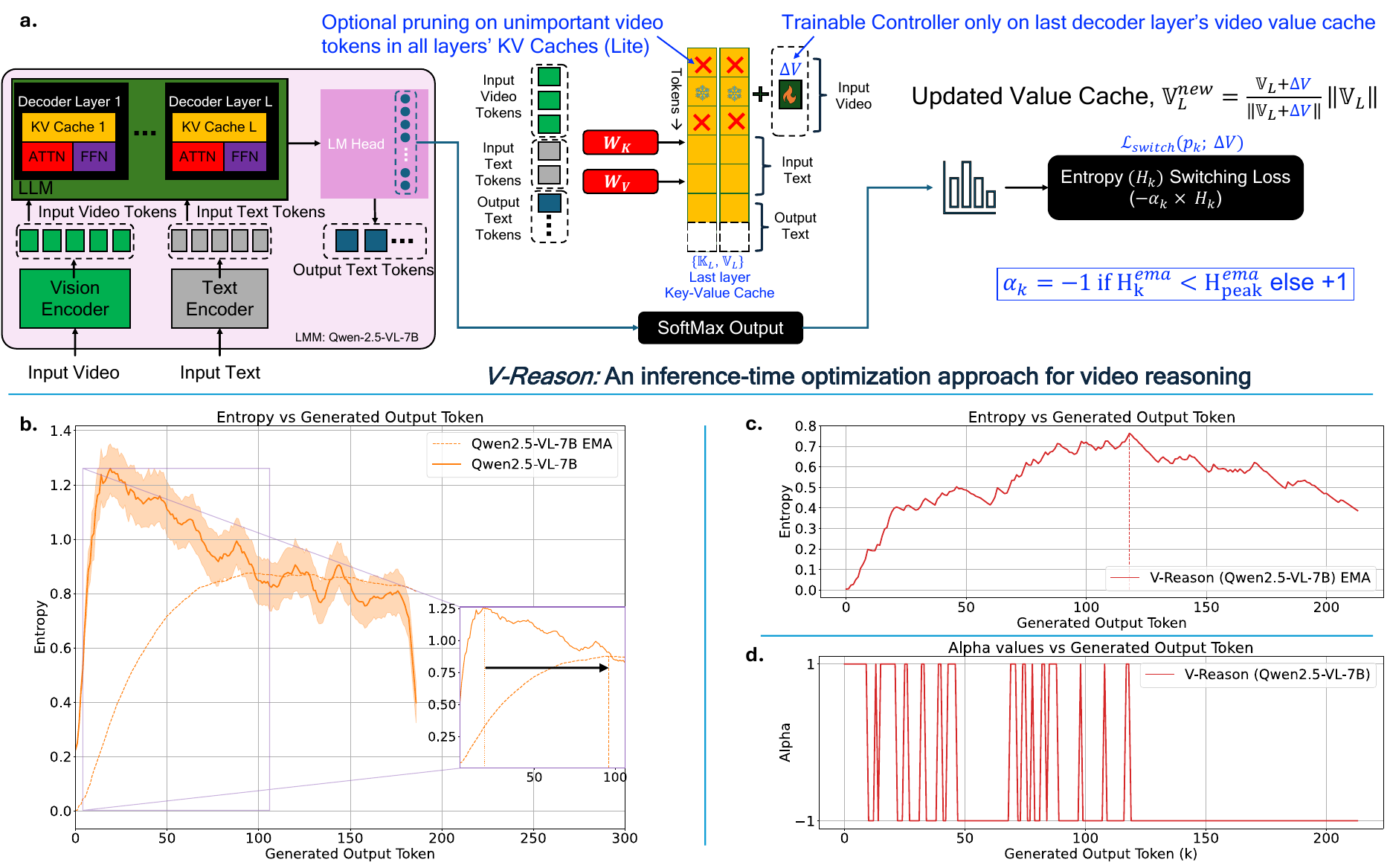}
   \caption{(a) Proposed approach for video reasoning in a training-free manner using entropy-based objective. \ours{} uses an inference optimization method to modulate the values cache of the last decoder layer with an entropy switching loss ($\mathcal{L}_{switch}$) to improve the video reasoning performance. (b) The average entropy plot for Qwen-2.5-VL-7B on the MMVU dataset along with its EMA. The inset depicts the shift in the entropy maxima for the EMA curve denoted by the black arrow (c) EMA entropy plot of \ours{} for a single sample that shows the micro-exploration and micro-exploitation within the macro-exploration phase before the entropy maxima and macro-exploitation phase after. (d) Plot showing the $\alpha_k$ switching in \ours{} for the corresponding example in (c) that ensures bounded entropy updates without a rapid increase.} 
   \label{fig:approach}
\end{figure*}

\subsection{Inference-time Optimization} 

Modifying the reasoning behavior of a pre-trained LMM requires two components: a set of parameters which are modified or added to the model to elicit reasoning, and an optimization objective, to optimize those parameters. As discussed in Section~\ref{sec:intro}, the key goals for \ours{} are to: (a)~decrease the rate of growth of the output distribution entropy during macro-exploration, by \textit{controlling} the model behavior so as to promote more pronounced cycles of micro-exploration and micro-exploitation during the output generation, and (b)~reduce the final entropy during macro-exploitation. To accomplish these objectives, we propose a value-cache controller and a novel inference-time optimization objective.

\noindent\textbf{Controller Parameters.} We propose to augment the model with the {\it Value-Cache Controller} shown in Fig.~\ref{fig:approach}(a). This controller, denoted as $\Delta V$, is a small, trainable parameter added to the value cache $\mathbb{V}_L$ of the \textit{last} decoder layer of the model, specifically at the video token locations. All other model layers remain frozen and no modifications are applied to the input or output text tokens. In our experiments, tuning only the value cache of the last layer was sufficient and modifying earlier layers or tuning key-value cache did not yield additional improvements. The controller $\Delta V$ is initialized to zero and updated at every $k^{th}$ generated output token ($k > 1$) via the inference-time optimization method discussed below. Note that no optimization is performed for the first token, as that is when the KV-Cache prefilling happens for all layers.
To prevent the controller from destabilizing the pretrained model, we introduce the normalization
\begin{equation}
    \mathbb{V}^{new}_L = \frac{\mathbb{V}_L + \Delta V}{||\mathbb{V}_L + \Delta V||} \cdot ||\mathbb{V}_L||.
    \label{eq:norm}
\end{equation}
This normalization preserves the original magnitude $||\mathbb{V}_L||$ of the cache vector, ensuring that the controller $\Delta V$ introduces only a directional update. This helps maintain a stable forward pass, ensuring consistent output token generation. This normalization is inspired by well-known methods like Weight Normalization~\citep{salimans2016weight, srebro2005rank}, which have been shown to have good optimization properties and are beneficial for recurrent and generative models.

\noindent\textbf{Optimization Objective.} In Section~\ref{sec:intro} and Fig.~\ref{fig:entropyIntro}(a), we suggested that the effectiveness of a reasoning model is related to the entropy of its output token distribution. While all reasoning models exhibit a period of macro-exploration, where entropy increases, and macro-exploitation, where it decreases, better models have a macro-exploration stage characterized by lower and delayed entropy maxima. We further posited that this is largely driven by cycles of micro-exploration and micro-exploitation, which prevent the entropy from increasing or decreasing too rapidly.
We interpret these cycles as periods where the model temporarily increases the output entropy (exploration) to allow alternative reasoning paths, needed to escape from a  current unpromising path. The model then pursues a new path in more detail (exploitation), leading to a decrease of entropy and the potential realization that this new path is itself not promising. The cycle is then repeated. We hypothesize that stronger reasoning models are more decisive in their patterns of micro-exploration and exploitation, which leads to more and/or stronger cycles, thus reducing the rate of {\it macro} entropy increase. This leads to lower and delayed entropy peaks. It follows that the reasoning power of a model should increase if it is encouraged to have more {\it vigorous} micro-exploration/exploitation cycles. After reaching the entropy peak of macro-exploration, the model switches to macro-exploitation, where it pursues a reasoning path in detail to produce an answer, which leads to a decrease of the output entropy. We note that the \textit{pattern of delayed peaks and lower final entropy and richer micro-cycles are signatures of better reasoning than absolute entropy values}.
In this work, we propose to reinforce this behavior by optimizing the value cache controller $\Delta V$ with the {\it Entropy Switching Loss}:
\begin{equation}\label{eq:entropy_loss}
    \begin{aligned}
 \mathcal{L}_{switch}( \Delta V)\  & = -\alpha_k H_k  = \alpha_k \sum_{i\in |\mathcal{V}|}p^i_k(\Delta V) \text{log}( p^i_k(\Delta V))
    \end{aligned}
\end{equation}
where $p_k$ is the output distribution (softmax after the LM-Head) for every $k^{th}$ token generated ($k>1$), $H_k$ the entropy of this distribution, and $\alpha_k\in \{-1, +1\}$ is a coefficient that switches between $-1$ and $+1$. The minimization of this loss encourages an increase in the entropy (micro-exploration) when $\alpha_k =1$ and a decrease (micro-exploitation) when $\alpha_k =-1$. Hence, setting $\alpha_k =1$ ($\alpha_k =-1$) during the micro-exploration (micro-exploitation) periods, encourages the model to be more decisive in its micro-exploration/exploitation cycles. It is also possible to explore other behaviors, e.g., using this procedure to reinforce micro-cycles during macro-exploration, followed by minimizing entropy alone ($\alpha_k = -1$) during macro-exploitation. 

To implement this, we first compute the exponential moving average (EMA) of the entropy at each generation step $t$ (different from $k$, which is the optimization step for the value-cache controller)
\begin{equation}
H_t^{ema} = \beta H_{t-1} + (1-\beta)H_t
\end{equation}
where $t > 1$, $\beta$ is a smoothing coefficient (set to 0.98), and $H_0$ is the entropy of the first token which is a small value\footnote{The baseline instruction-tuned models are certain about the very first predicted token; it is usually just the \texttt{<think>} token, even without RL or CoT-SFT, because of the instruction we give to the model.}. The EMA is a low-pass filtered version of the raw entropy, and thus much less noisy, as shown in Fig.~\ref{fig:approach} (b). It achieves a good trade-off between oscillating too much, due to noise, and switching between increasing and decreasing entropy during micro-cycles, as shown in Fig.~\ref{fig:approach} (c). 
Also, because it grows much slower than the raw entropy, following the EMA naturally leads to a lower and delayed entropy peak, as shown in Fig.~\ref{fig:approach} (b). The switching coefficient  $\alpha_k$ is then defined to follow the EMA, 
\begin{equation}
    \alpha_k = 
    \begin{cases}
        +1 \ \text{if} 
        H_k^{ema} \geq H_{peak}^{ema}\\
        -1 \ \text{if} \ H_k^{ema} < H_{peak}^{ema}
    \end{cases}\\
    \label{eq:alphak}
\end{equation}
where, $H_k^{ema}$ is the EMA at the current step, and  $H_{peak}^{ema}$ the  maximum value of EMA observed before step $k$. This is illustrated in Figure~\ref{fig:approach} (d). It encourages the entropy to (\textit{i})~increase when the EMA is larger than the last peak, i.e., the EMA is increasing, and to (\textit{ii})~decrease otherwise, i.e., the EMA is decreasing, therefore reinforcing the natural micro-cycles of the model.
Once the EMA reaches a global maximum, $\alpha_k$ becomes $-1$ and macro-exploitation begins. This global maximum of entropy can also be seen as a more formal definition of the end of the ``thinking'' process. A detailed description of our method is given in Algorithm 1.

Fig.~\ref{fig:approach} (c) shows the EMA entropy plot of \ours{} for a single sample. It is clear that there are more and stronger local minima and maxima depicting the micro-exploration/exploitation cycles before the entropy maxima. This slows the entropy growth during macro-exploration, leading to a delayed peak and substantially more exploration than by the original model. Once  the global maximum of the EMA is reached, $\alpha_k$ becomes $-1$ and the model enters the macro-exploitation stage, where it is encouraged to decrease entropy until it arrives at a solution. Overall, the optimization promotes 1) more and/or longer cycles of micro-exploration and micro-exploitation during the macro-exploration stage, which lead to ``longer thinking,'' with lower and delayed entropy peaks, and 2) a stronger emphasis on entropy minimization during the macro-exploitation stage, which leads to faster convergence to a lower final entropy. 

We observe that the optimization of \ours{} encourages the model to arrive at the final solution significantly faster than CoT-SFT and RL models, which often produce verbose outputs. This can be seen in Fig.~\ref{fig:entropyIntro} (d). Since computation is tied to the length of the output sequence, this also results in significantly more efficient inference than those models. Hence, despite the extra computation needed for the optimization, \ours has more efficient inference overall (section~\ref{sec:eff}).
Finally, since \ours{} exploits the natural variation in entropy, it adaptively determines how much exploration and exploitation is required by each sample. This makes it robust and adaptable to various datasets and types of video reasoning problems (see Section~\ref{sec:exp}).

\noindent\textbf{Theoretical Guarantees.} We provide theoretical guarantees that the entropy updates induced by our Entropy Switching Loss remain stable and that our EMA-based objective bounds the oscillations in entropy. The formal statements are below, with assumptions and proofs discussed in Supplementary A. 
\begin{proposition}[Bounded entropy updates]
\label{prop:boundedness_main}
Under mild smoothness and boundedness assumptions, one gradient step of size $\eta$ on the Entropy Switching Loss changes entropy by at most
\[
|H_{t+1}-H_t| \le \eta C + o(\eta),
\]
and the process $\{H_t\}$ remains within the compact interval $[0,\log n]$.
\end{proposition}

\begin{proposition}[EMA smoothing bounds oscillations]
\label{prop:ema_main}
For $\beta\in(0,1)$ close to $1$, the EMA acts as a low-pass filter: \emph{(\textit{i})}~it attenuates high-frequency fluctuations of $H_t$, \emph{(\textit{ii})}~delays the attainment of entropy maxima, and \emph{(\textit{iii})}~enforces bounded oscillations by switching $\alpha_k$ to $-1$ once a new global EMA maximum is reached.
\end{proposition}

\subsection{Efficiency Considerations: \textbf{\ourslite{}}}\label{sec:eff}
Video reasoning with LMMs can have high GPU memory costs due to a large number of input video tokens. Adding inference-time optimization to these models at first sight can seem inefficient, as it can further increase inference costs. However, \ours{} has several properties that counteract this hypothesis. First, the controller is only added to the decoder cache of the last model layer. This significantly reduces the memory overhead of storing activations for backpropagation, which reduces to the trainable controller $\Delta V$ and a few feature maps (last decoder layer's value cache, attention output, feedforward layers, and LM-Head). Second, and most important, because \ours{} usually arrives at the final solution with significantly less tokens as shown in Fig.~\ref{fig:entropyIntro}(d), both its inference time and computation are much lower than models trained to think.

Nevertheless, we explore an additional avenue for efficiency. Before performing the \ours{} optimization, we \textit{optionally} prune 50\% of the video tokens from the KV-Cache of all decoder layers, a variant we refer to as \ourslite{}. This significantly reduces the KV-Cache overhead and also halves the size of the trainable controller. Interestingly, we found that for some datasets this also slightly improves \ours{} reasoning performance (perhaps by reducing noise due to unimportant video tokens). To prune out unimportant video tokens, we measure the mean value of the $l_2$ norm of video tokens across all value caches and eliminate the lowest 50\% video tokens from both Key and Value Caches of all decoder layers. The trainable controller is then only added to the remaining video tokens in the last decoder layer. The new value update is $\mathbb{V}^{new}_L = \frac{\mathbb{V}^{pruned}_L + \Delta V}{||\mathbb{V}^{pruned}_L + \Delta V||} \cdot ||\mathbb{V}_L||$,
which still maintains the magnitude of the unpruned video value cache from~\eqref{eq:norm}. We empirically find that this reduces the error due to pruning and enables the \ourslite{} models to achieve much higher accuracies than when the value cache norm is altered. Algorithm 2 in Supplementary provides the pseudo-code for the lite variant.

\section{Experiments}\label{sec:exp}
\noindent\textbf{Implementation Details.}\label{sec:implementation_main}
All experiments use pytorch version \texttt{2.5.1+cu121}, transformers version \texttt{4.52.4}, and a single NVIDIA-A100 GPU. Following~\citep{feng2025videor1}, we use multinomial sampling with (\texttt{temperature=0.1}, \texttt{top-p=0.001} i.e., deterministic) for our experiments unless otherwise noted. See Supplementary B for more details.

\noindent\textbf{Video Reasoning.} We evaluate \ours{} on the Qwen-2.5-VL-Instruct~\citep{Qwen2.5-VL} model series under 16/32 frames settings (from~\citep{feng2025videor1}) and maximum video pixels px$\times28\times28$ with px=256/128, respectively. Similar to~\citep{feng2025videor1}, \ours is evaluated across 6 video reasoning benchmarks,  covering two tasks, Multiple-Choice QA and Regression, evaluated by classification accuracy and Mean Relative Accuracy (MRA) respectively. We report the average accuracy with and without MRA to illustrate the model’s performance across different task formulations.

\subsection{Video Reasoning Benchmark Results}

\begin{table*}[t]
\centering
\scriptsize
\caption{Comparison of performance of different models on video reasoning benchmarks. \#F denotes the number of frames and px denotes the maximum video pixels used, px$\times28\times28$.}
\setlength{\tabcolsep}{2pt}
\resizebox{\textwidth}{!}{%
\begin{tabular}{l c c c c c c c c c}
\toprule
\textbf{Model} & \textbf{\#F/px} & \textbf{VSI-Bench} & \textbf{VideoMMMU} & \textbf{MMVU} & \textbf{MVBench} & \textbf{TempCompass} & \textbf{VideoMME} & Avg & Avg\\
& & (Acc/MRA) & & (mc) & & & (wo sub) &  & (wo mra) \\
& & ~\citep{yang2025thinking_vsibench} & ~\citep{hu2025videommmu}& ~\citep{zhao2025mmvu} & ~\citep{li2024mvbench}& ~\citep{liu2024tempcompass}& ~\citep{fu2024videomme} &  & \\
\midrule
GPT-4o~\citep{openai2024gpt4o}  & -- & 34.0 & 61.2 & 75.4 & -- & -- & 71.9 & --& --\\
LLaMA-VID~\citep{li2023llamavid}  & -- & -- & -- & -- & 41.9 & 45.6 & -- & --& --\\
VideoLLaMA2~\citep{cheng2024videollama2}  & -- & -- & -- & 44.8 & 54.6 & -- & 47.9 & --& --\\
LongVA-7B~\citep{zhang2024longva}  & -- & 29.2 & 23.9 & -- & -- & 56.9 & 52.6 & --& --\\
VILA-1.5-8B~\citep{lin2023vila}  & -- & 28.9 & 20.8 & -- & -- & 58.8 & -- & --& --\\
Video-UTR-7B~\citep{lin2025videoutr}  & -- & -- & -- & -- & 58.8 & 59.7 & 52.6 & --& --\\
LLaVA-OneV-7B~\citep{li2024llavaonevision}  & -- & 32.4 & 33.8 & 49.2 & 56.7 & -- & 58.2 & --& --\\
\midrule

Qwen2.5-VL-3B~\citep{Qwen2.5-VL} & 32/128 & 24.3 (31.6/17.0) & 32.3 & 49.3 & 52.5 & 28.1 & 48.1 & 37.0 & 40.3 \\
\textbf{\ours{}-3B \texttt{(Lite)}} & 32/128 & \textbf{26.3 (32.2/20.4)} \textbf{\textcolor{ForestGreen}{[+0.6/+3.4]}} & \textbf{33.9} \textbf{\textcolor{ForestGreen}{[+1.6]}} & \textbf{50.9} \textbf{\textcolor{ForestGreen}{[+1.6]}} & \textbf{53.2} \textbf{\textcolor{ForestGreen}{[+0.7]}} & \textbf{29.1} \textbf{\textcolor{ForestGreen}{[+1.0]}} & \textbf{49.0} \textbf{\textcolor{ForestGreen}{[+0.9]}} & \textbf{38.3} \textbf{\textcolor{ForestGreen}{[+1.3]}} & \textbf{41.3} \textbf{\textcolor{ForestGreen}{[+1.0]}}\\

\textbf{\ours{}-3B} & 32/128 & \textbf{24.7 (31.9/17.5)} \textbf{\textcolor{ForestGreen}{[+0.3/+0.5]}} & \textbf{33.2} \textbf{\textcolor{ForestGreen}{[+0.9]}} & \textbf{50.2} \textbf{\textcolor{ForestGreen}{[+0.9]}} & \textbf{52.9} \textbf{\textcolor{ForestGreen}{[+0.4]}} & \textbf{30.4} \textbf{\textcolor{ForestGreen}{[+2.3]}} & \textbf{48.8} \textbf{\textcolor{ForestGreen}{[+0.7]}} & \textbf{37.9} \textbf{\textcolor{ForestGreen}{[+0.9]}} & \textbf{41.2} \textbf{\textcolor{ForestGreen}{[+0.9]}}\\

\midrule
Qwen2.5-VL-7B~\citep{Qwen2.5-VL} & 16/256 & 26.4 (31.4/21.4) & 47.6 & 59.5 & 60.4 & 72.2 & 50.5 & 49.0& 53.6\\
\textbf{\ours{}-7B \texttt{(Lite)}} & 16/256 & \textbf{27.9 (34.1/21.6)} \textbf{\textcolor{ForestGreen}{[+2.7/+0.2]}} & 47.6 \textbf{\textcolor{black}{[+0.0]}} & \textbf{63.4} \textbf{\textcolor{ForestGreen}{[+3.9]}} & \textbf{60.8} \textbf{\textcolor{ForestGreen}{[+0.4]}} & 71.6 \textbf{\textcolor{Crimson}{[-0.6]}} &  \textbf{51.1} \textbf{\textcolor{ForestGreen}{[+0.6]}} & \textbf{49.9} \textbf{\textcolor{ForestGreen}{[+0.9]}}& \textbf{54.6} \textbf{\textcolor{ForestGreen}{[+1.0]}}\\
\textbf{\ours{}-7B} & 16/256 & \textbf{28.5 (34.5/22.6)} \textbf{\textcolor{ForestGreen}{[+3.1/+1.2]}} & \textbf{47.8} \textbf{\textcolor{ForestGreen}{[+0.2]}} & \textbf{62.2} \textbf{\textcolor{ForestGreen}{[+2.7]}} & \textbf{61.0} \textbf{\textcolor{ForestGreen}{[+0.6]}} & \textbf{72.3} \textbf{\textcolor{ForestGreen}{[+0.1]}} & \textbf{51.1} \textbf{\textcolor{ForestGreen}{[+0.6]}} & \textbf{50.2} \textbf{\textcolor{ForestGreen}{[+1.2]}}& \textbf{54.8} \textbf{\textcolor{ForestGreen}{[+1.2]}}\\
\midrule
Video-R1-7B~\citep{feng2025videor1}  & 16/256 & \textbf{33.8} (30.5/\textbf{37.0}) & \textbf{47.8} & \textbf{64.2} & \textbf{63.9} &  72.2 &  \textbf{57.2} & \textbf{53.3}& \textbf{56.0}\\
\midrule
Qwen2.5-VL-7B~\citep{Qwen2.5-VL} & 32/128 & 28.1 (33.8/22.3) & 45.8 & 61.3 & 60.7 & 72.4 & 53.7 & 50.0& 54.6\\
\textbf{\ours{}-7B \texttt{(Lite)}} & 32/128 & \textbf{30.5 (37.3/23.7)} \textbf{\textcolor{ForestGreen}{[+3.5/+1.4]}} & \textbf{47.4} \textbf{\textcolor{ForestGreen}{[+1.6]}} & \textbf{65.0} \textbf{\textcolor{ForestGreen}{[+3.7]}} & \textbf{60.6} \textbf{\textcolor{Crimson}{[–0.1]}} & \textbf{72.4} \textbf{\textcolor{black}{[+0.0]}} & \textbf{53.5} \textbf{\textcolor{Crimson}{[–0.2]}} & \textbf{51.4} \textbf{\textcolor{ForestGreen}{[+1.4]}} & \textbf{56.0} \textbf{\textcolor{ForestGreen}{[+1.4]}}\\

\textbf{\ours{}-7B} & 32/128 & \textbf{30.3 (37.1/23.4)} \textbf{\textcolor{ForestGreen}{[+3.3/+1.1]}} & \textbf{46.3} \textbf{\textcolor{ForestGreen}{[+0.5]}} & \textbf{62.7} \textbf{\textcolor{ForestGreen}{[+1.4]}} & \textbf{60.9} \textbf{\textcolor{ForestGreen}{[+0.2]}} & \textbf{73.3} \textbf{\textcolor{ForestGreen}{[+0.9]}} & \textbf{54.9} \textbf{\textcolor{ForestGreen}{[+1.2]}} & \textbf{51.2} \textbf{\textcolor{ForestGreen}{[+1.2]}} & \textbf{55.9} \textbf{\textcolor{ForestGreen}{[+1.3]}}\\

\midrule
Video-R1-7B~\citep{feng2025videor1}  & 32/128 & \textbf{35.6} (30.9/\textbf{39.2}) & \textbf{48.8} & \textbf{64.0} & \textbf{64.1} & \textbf{73.3} &  \textbf{58.7} & \textbf{54.1}& \textbf{56.6}\\ %
\bottomrule
\label{tab:video_comparison}
\end{tabular}
}
\end{table*}

Table~\ref{tab:video_comparison} presents a comparison of \ours{} with Qwen2.5-VL-Instruct baselines and the RL-trained Video-R1-7B across multiple video reasoning benchmarks. Green brackets show the gain of the \ours{} model over the baseline, with negative gains in red. Both (at least one) versions of \ours{} improve the baseline performance for 15/18 (18/18) model/dataset combinations.  Furthermore, the gain is of at least \textbf{0.9 points} for 19/36 combinations and can be as high as \textbf{3.9 points}. In many cases, these gains are a substantial part of the gap between the baseline and the RL-trained model. For example, for MMVU and 7B-256px models the 63.4 point accuracy of \ours{} (Lite) brings the relatively lower 59.5 point baseline close to the 64.2 point accuracy of the Video-R1. For the 128 px model, \ours{} even surpasses Video-R1 (\textbf{65.0 vs. 64.0}). This model also matches Video-R1 on TempCompass (73.3 each), and nearly closes the gap on VideoMMMU (47.4 vs. 48.8). 
These very significant gains show that the baseline model already has a significant ability to reason, which RL brings to the surface, but can  also be mostly unlocked by much less expensive inference-time optimization of \ours{}.  

We empirically observe that our training-free, search-based method is most effective when the base model already possesses the relevant underlying knowledge. For certain tasks like VSI-MRA which are underrepresented in the model's learned knowledge space, \ours{} still obtains a \textbf{+1.4\%} improvement over the baseline. 

Overall, across model scales and input resolutions, \ours{} and \ourslite{} have average gains in \textbf{[+0.9,1.3]\%}  over Qwen2.5-VL, at the 3B scale. At the 7B scale, \ours{}/\ourslite{} reaches 54.8\%/54.6\% (256 px) and 55.9\%/56.0\% (128 px) average accuracy without MRA, narrowing the gap to Video-R1-7B (56.6\% at 128 px) to within \textbf{0.6\%}.

\begin{figure*}[t]
\centering
\includegraphics[keepaspectratio, width=0.9\textwidth]{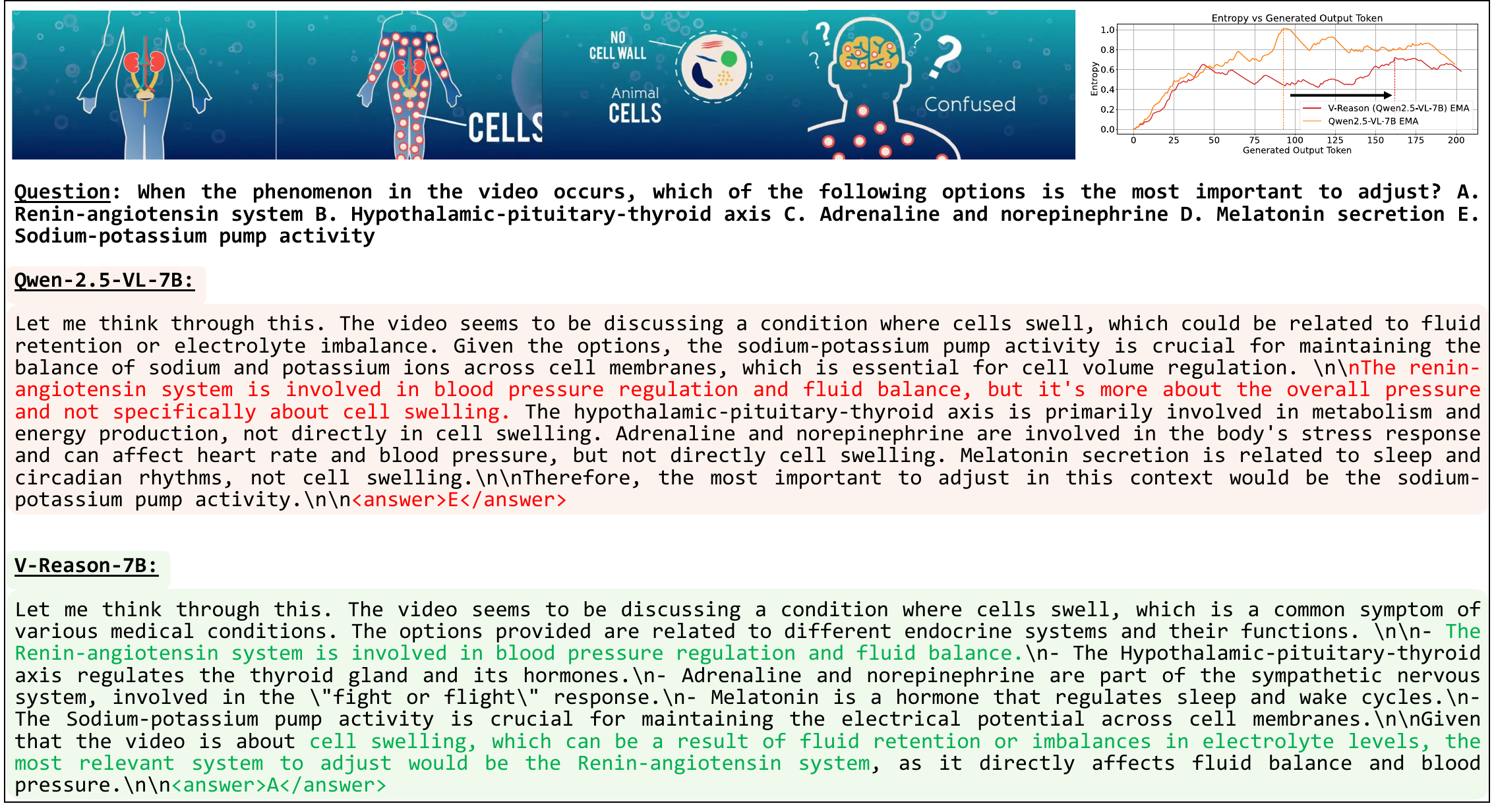}
\caption{Qualitative result: An example output and comparison with the baseline Qwen-2.5-VL-7B together with its entropy plot shown on the top right. The black arrow in the entropy plot denotes the shift in the EMA peak demonstrating longer exploration for \ours{} compared to the baseline.}
\label{fig:qual_adibatic}
\end{figure*}

\noindent\textbf{Impact of Frames and Resolution.}
\ours{} is robust to different frame counts and resolutions, making it adaptable to resource-constrained settings. For a comparable compute budget (256 px/16 frames vs. 128 px/32 frames), accuracy improves with more frames. Consistent with this trend, \ours{} shows larger average gains at 32 frames (\textbf{+1.4\%}) compared to 16 frames (\textbf{+1.0\%}).

\noindent\textbf{Full vs. Lite.}
Both Full and Lite \ours{}  variants surpass the base model, with Lite performing comparably or slightly better in several cases (e.g., \textbf{+1.0\%} at 3B and \textbf{+1.4\%} at 7B for 128 px). The Full variant offers an additional \textbf{+0.2\%} improvement at 256 px with 16 frames, suggesting that structural pruning is complementary to reasoning. 

\noindent\textbf{Output Sequence Length.} As shown in Fig.~\ref{fig:entropyIntro}(d), \ours{} substantially reduces output token length (\textbf{58.6\% reduction} over Video-R1). Table 5 (Supplementary C) shows that this translates into significant latency savings of up to \textbf{67\%} (\textbf{37\%} on average). The fact that this holds for both Full and Lite variants indicates that the gains stem from the proposed optimization rather than pruning.

To test the scalability of \ours{}, we further evaluate it on Qwen3-VL and larger Qwen2.5-VL backbones (32B and 72B) using  MMVU.
Table~\ref{tab:mmvu_comparison} shows that gains hold across model families: e.g., \textbf{+5.1\%} over MMVU for Qwen3-VL, showing that the method is not architecture-specific. \ours{} continues to provide significant gains on larger model sizes, e.g., \textbf{+3.0\% (72.0 vs. 69.0)} for the 32B model, demonstrating that reasoning benefits do not diminish with scale. For the 72B model, which is already strong, \ours{} still yields a \textbf{+0.4\% gain (73.0 vs. 72.6)}.
These results show that \ours{} generalizes to different backbones and larger models and provides benefits even at the frontier of large-scale video reasoning models.

\begin{table}[t]
\centering
\scriptsize
\setlength{\tabcolsep}{2pt}

\begin{minipage}[t]{0.47\columnwidth}
\centering
\caption{Large model results on MMVU.}
\label{tab:mmvu_comparison}
\begin{tabular}{l|c}
\toprule
\textbf{Model} & \textbf{MMVU} \\
\midrule
Qwen-2.5-VL-32B & 69.0 \\
\ours{}-32B & \textbf{72.0 \textcolor{ForestGreen}{[+3.0]}} \\
Qwen-2.5-VL-72B & 72.6 \\
\ours{}-72B & \textbf{73.0 \textcolor{ForestGreen}{[+0.4]}} \\
\midrule
Qwen-3-VL-8B & 64.8 \\
\ours{}-8B & \textbf{69.9 \textcolor{ForestGreen}{[+5.1]}} \\
\bottomrule
\end{tabular}
\end{minipage}
\hfill
\begin{minipage}[t]{0.47\columnwidth}
\centering
\caption{Optimization objective \\ablations.}
\label{tab:ablation}
\begin{tabular}{l|c}
\toprule
\textbf{Method} & \textbf{MMVU} \\
\midrule
Qwen-2.5-VL-7B & 61.3 \\
Min-Entropy \texttt{(Lite)} & 62.1 \textcolor{ForestGreen}{[+0.8]} \\
Max-Entropy \texttt{(Lite)} & 63.8 \textcolor{ForestGreen}{[+2.5]} \\
KV Cache \texttt{(Lite)} &  61.8\textcolor{ForestGreen}{[+0.5]} \\
Last two \texttt{(Lite)} &  62.2\textcolor{ForestGreen}{[+0.9]} \\
\midrule
\ourslite{} & \textbf{65.0 \textcolor{ForestGreen}{[+3.7]}} \\
\bottomrule
\end{tabular}
\end{minipage}

\end{table}

\noindent\textbf{Comparison with Decoding Methods.}
\begin{figure} 
    \centering
    \scriptsize
    \setlength{\tabcolsep}{2pt}
    \captionof{table}{Comparison with alternative decoding methods.}
    \label{tab:comparison_decoding}
    \begin{tabular}{l|c|c|c}
    \toprule
    \textbf{Qwen-2.5-VL-7B} & \texttt{temp} & \texttt{top-p}& \textbf{MMVU} \\
    \midrule
    min-p & 0.3 & 0.9& 61.8 \\
    min-p+\ourslite{} &0.3 & 0.9& \textbf{63.8 \textcolor{ForestGreen}{[+2.0]}} \\
    top-H &0.3 & 0.9& 60.2 \\
    top-H+\ourslite{} &0.3 & 0.9& \textbf{61.1 \textcolor{ForestGreen}{[+0.9]}} \\
    \midrule
    min-p & 1.0 & 0.9& 55.0 \\
    min-p+\ourslite{} &1.0 & 0.9& \textbf{61.3 \textcolor{ForestGreen}{[+6.3]}} \\
    top-H &1.0 & 0.9& 62.2 \\
    top-H+\ourslite{} &1.0 & 0.9& \textbf{62.6 \textcolor{ForestGreen}{[+0.4]}} \\
    \midrule
    \ours{}-7B \texttt{(Lite)} & 0.1 & 0.001&\textbf{65.0 \textcolor{ForestGreen}{[+2.8]}} \\
    \bottomrule
    \end{tabular}
\end{figure}
As shown in Table~\ref{tab:comparison_decoding}, our method is robust and complementary to different decoding strategies with significant improvements over SOTA approaches such as \textit{min-p}~\citep{Nguyen2024MinPSampling} and \textit{top-H}~\citep{Potraghloo2025TopHDecoding}. 
For the Qwen-2.5-VL-7B model, using the best \textit{min-p} decoding with \ourslite{} yields a gain of {\textbf{+2.0}} points on MMVU, while combining with best \textit{top-H} decoding provides a smaller improvement of \textbf{{+0.4}}. On higher temperatures, \textit{min-p} loses significant accuracy but \ourslite{} is able to restore it back (\textbf{+6.3\%}). Most notably, \ours{}-7B \texttt{(Lite)} achieves the highest score of \textbf{65.0}, corresponding to a further \textbf{+2.8} gain over the best decoding baseline.

\noindent\textbf{Qualitative Results.}
Figure~\ref{fig:qual_adibatic} exemplifies the reasoning differences between \ours{} and the baseline, also showing their entropy profiles. The entropy plots reveal that \ours{} has a delayed EMA peak and a lower overall entropy, encouraging extended exploration that ultimately enables the model to reach the correct solution. As highlighted in red, the baseline initially follows a promising trajectory but subsequently diverges onto an incorrect reasoning path, which leads to the wrong answer. In contrast, \ours{} identifies an alternative path precisely at the point where the baseline falters, and this revised trajectory, shown in green, successfully leads to the correct answer. Please see Supplementary I for other examples.

\noindent\textbf{Alternative Losses.} 

The switching loss in~\eqref{eq:entropy_loss} supports various behaviors beyond that encouraged by \ours{}. 
Two extreme alternatives are enforcing strictly increasing entropy (max-entropy, $\alpha_k =1, \forall k$) and strictly decreasing entropy (min-entropy, $\alpha_k =-1, \forall k$). Table~\ref{tab:ablation} shows the two alternative losses, additional ablations on updating a KV cache controller and the last two layers of the model. It shows that these approaches are clearly inferior to \ours{}. However, it is interesting to note that even the basic losses (encourage macro-exploration or macro-exploitation only) improve on the performance of the baseline model. This confirms that structured entropy control helps with the reasoning ability of LMMs.

\vspace{-3mm}
\section{Conclusion}\label{sec:conc}
\vspace{-1mm}
In this paper, we introduced \ours{}, a training-free framework that elicits video reasoning through a value-cache controller at inference. Our method leverages a theoretically-grounded entropy-based objective to reinforce the micro-exploration and micro-exploitation behaviors observed across models. This design effectively mitigates unbounded entropy growth during early generation steps, resulting in delayed entropy peak and lower final entropy, a characteristic of stronger models. We further proposed \ourslite{}, a ``Lite'' variant which improves the memory by pruning low $l_2$-norm entries in the value cache.
Extensive experiments across multiple benchmarks demonstrate that \ours{} narrows the gap to RL–trained models (e.g., Video-R1) to within \textbf{0.6\%}, while reducing output token length \textbf{($\downarrow$58.6\%)}; this also results in lower \textbf{($\downarrow$37\%)} inference time than Video-R1. Moreover, \ours{} consistently improves performance across model scales ranging from 3B to 72B parameters and remains robust to variations in frame sampling, pixel resolution, decoding techniques, and other hyperparameter configurations.

{
    \small
    \bibliography{main}
}
\clearpage
\appendix

\section{Theoretical Analysis: Bounding Entropy under Switching Loss}
\label{sec:theory_entropy_boundedness}
\setcounter{proposition}{0}
Let the vocabulary size be $n=|\mathcal V|$. At generation step $t$, the model (with value-cache controller parameters $\Delta V$) produces logits $z_t\in\mathbb R^n$ and probabilities
\[
p_t(\Delta V) = \mathrm{softmax}(z_t(\Delta V)),\qquad \sum_i p_t^i=1.
\]
The Shannon entropy of this distribution is
\[
H_t(\Delta V) := -\sum_{i=1}^n p_t^i(\Delta V)\log p_t^i(\Delta V),
\]
and its exponential moving average (EMA) is
\[
H_t^{ema} = \beta H_{t-1}^{ema} + (1-\beta)H_t,\quad \beta\in(0,1).
\]

The Entropy Switching Loss at optimization step $k$ is
\[
\mathcal L_{switch}(\Delta V) = -\alpha_k H_k(\Delta V),
\]
where the coefficient $\alpha_k\in\{-1,+1\}$ is defined as
\[
\alpha_k =
\begin{cases}
+1 & \text{if } H_k^{ema} \ge H_{peak}^{ema},\\
-1 & \text{otherwise},
\end{cases}
\]
with $H_{peak}^{ema}$ denoting the maximum EMA value observed before step $k$.

\paragraph{Assumptions.} We make the following assumptions:
\begin{enumerate}
    \item Logits $z_t(\Delta V)$ are smooth in $\Delta V$, and $\partial z_t/\partial \Delta V$ is bounded. From equation 1, $\mathbb{V}^{new}_L = \frac{\mathbb{V}_L + \Delta V}{||\mathbb{V}_L + \Delta V||} \cdot ||\mathbb{V}_L||.$ So, $\partial z_t/\partial \Delta V$ being bounded is a valid assumption because the update to value cache is bounded by the normalization factor which only provides a directional update. 
    \item The optimizer uses a bounded step size (learning rate) $\eta>0$ and updates are sufficiently small per step (i.e., standard stochastic gradient/Lipschitz assumptions).
    \item Vocabulary size is finite, hence $H_t\in[0,\log n]$ for all $t$.
\end{enumerate}

\paragraph{Preliminaries.}
Differentiating the entropy with respect to logits yields
\[
\nabla_z H = -J_p^\top (\mathbf{1}+\log p),
\]
where $J_p = \partial p/\partial z$ is the softmax Jacobian. Since $\|J_p\|$ is bounded and $\mathbf{1}+\log p$ is finite (as $p_i\in(0,1]$), we obtain
\[
\|\nabla_{\Delta V} H\| \le C
\]
for some constant $C$.

\begin{proposition}[Bounded entropy updates]
\label{prop:boundedness}
Under the assumptions above, one gradient step of size $\eta$ on $\mathcal L_{switch}$ changes entropy by at most
\[
|H_{t+1}-H_t| \le \eta C + o(\eta),
\]
and the process $\{H_t\}$ remains in the compact interval $[0,\log n]$. Here, $o(\eta)$ denotes the higher-order terms from the Taylor expansion of $H(\Delta V)$ around the current iterate.
\end{proposition}

\begin{proof}
First, the gradient of entropy with respect to controller parameters is
\[
\nabla_{\Delta V} H_k(\Delta V) = \frac{\partial H_k}{\partial z_k} \frac{\partial z_k}{\partial \Delta V}.
\]

\textbf{Bounding $\nabla_z H_k$.}  
For softmax probabilities bounded away from $0$ and $1$, the Jacobian $J_p = \partial p_k / \partial z_k$ satisfies $\|J_p\|_2 \le 1$. Moreover, the entropy gradient w.r.t. logits is
\[
\nabla_z H_k = -J_p^\top (\mathbf{1} + \log p_k),
\]
and $\|\mathbf{1} + \log p_k\|_2 \le \sqrt{n} \max_i |1 + \log p_k^i| \le C_1$ for some constant $C_1$ depending on $n$ and $\epsilon$ (the lower bound on softmax probabilities). Therefore,
\[
\|\nabla_z H_k\|_2 \le C_1.
\]

\textbf{Bounding $\nabla_{\Delta V} H_k$.}  
Since $z_k$ is $L_z$-Lipschitz in $\Delta V$,
\[
\|\nabla_{\Delta V} H_k\|_2 = \|\nabla_z H_k \cdot \partial z_k / \partial \Delta V\|_2 \le C_1 L_z := L_H.
\]

\textbf{Bounding one gradient step.}  
A single gradient step updates the controller:
\[
\Delta V \gets \Delta V + \eta \alpha_k \nabla_{\Delta V} H_k.
\]
Using the Lipschitz property of $H_k$ w.r.t $\Delta V$,
\[
|H_k(\Delta V + \eta \alpha_k \nabla_{\Delta V} H_k) - H_k(\Delta V)| \le \eta \|\nabla_{\Delta V} H_k\|_2 \le \eta L_H.
\]

\textbf{Global bounds.}  
Since $H_k \in [0, \log n]$ by definition, this step-size bound guarantees the entropy remains in $[0,\log n]$ after each update.

\end{proof}

\begin{proposition}[EMA smoothing bounds oscillations]
\label{prop:ema}
For $\beta\in(0,1)$ close to $1$, the EMA acts as a low-pass filter: \emph{(\textit{i})}~it attenuates high-frequency fluctuations of $H_t$, \emph{(\textit{ii})}~delays the attainment of entropy maxima, and \emph{(\textit{iii})}~enforces bounded oscillations by switching $\alpha_k$ to $-1$ once a new global EMA maximum is reached.
\end{proposition}

\begin{proof}[Proof]
(\textit{i})~The recursion $H_t^{ema}=\beta H_{t-1}^{ema}+(1-\beta)H_t$ is a causal low-pass filter, suppressing fast oscillations. (\textit{ii})~Because $H^{ema}$ averages over past values, peaks in $H_t$ appear later and at lower amplitude in $H^{ema}$, creating delayed switching. (\textit{iii})~Once $H^{ema}$ reaches a global maximum, $\alpha=-1$, turning the loss into an entropy-minimization objective. This guarantees the entropy trajectory descends after each peak, bounding the amplitude of oscillations.
\end{proof}

\paragraph{Discussion.}
The trivial upper bound $H_t\le\log n$ already prevents unbounded entropy; Proposition~\ref{prop:boundedness} strengthens this by showing the optimization dynamics cannot instantaneously jump arbitrarily close to $\log n$ provided the learning rate is small and gradients are bounded. In practice, this prevents pathological ‘‘entropy blow-ups’’ during optimization. EMA smoothing makes the switching decision depend on sustained increases in entropy rather than on single noisy spikes. These results imply that the Entropy Switching Loss enforces \emph{bounded micro-cycles} of exploration and exploitation: entropy increases are promoted only when sustained (captured by $H^{ema}$), while decreases are enforced once a peak is reached. This yields lower and delayed entropy maxima, consistent with the empirical patterns of stronger reasoning models.

Concurrent work, Top-H~\citep{Potraghloo2025TopHDecoding}, formalizes entropy bounds in the decoding step by solving (approximately) an entropy-constrained minimization problem that upper-bounds the randomness of the truncated distribution while keeping divergence from the model distribution small. 
Our approach uses a complementary perspective: rather than imposing a hard constraint on the sampling distribution at each decoding step, we \emph{optimize the controller} so that the model’s intrinsic token distributions themselves enter phases of controlled exploration and exploitation (via maximizing/minimizing $H$ at different times). The EMA-based switching mirrors the time-adaptive, entropy-aware thresholds used in Top-H while operating \emph{inside} the model (controller optimization) rather than as an external truncation rule. Empirically and theoretically, both approaches rely on the same fundamental fact: \emph{entropy is a natural, bounded quantity} that can be used as a control signal to trade-off diversity and consistency in generation.

\section{Implementation Details}\label{app:details}
\paragraph{Hyperparameters.} 
AdamW optimizer is used to update the controller with no weight decay. A step size of $k=4$ is used as default unless otherwise specified and the best accuracy is reported over a grid search of 10 learning rates from \texttt{5e-5} to \texttt{5e-4}. The gradient norm of the value-cache controller is clipped to 1.0. We used $\beta = 0.98$ smoothing factor for EMA.

\paragraph{Evaluation.} Classification accuracy is computed as the proportion of correct answers to the multiple-choice QA. Mean Relative Accuracy measures the proportion of predictions whose relative error falls below a series of thresholds ranging from 0.5 to 0.95. The final score is the average accuracy across all thresholds. For VSI-Bench, we report both classification accuracy and MRA individually, as well as their average. To compute the overall average accuracy across all six datasets, we divide by seven, treating the two scores from VSI-Bench separately in addition to the other datasets. When calculating the average accuracy without considering MRA, we divide by six, using only the accuracy score from VSI-Bench along with the scores from the remaining datasets.
\vspace{-2mm}

\begin{algorithm}[h]
\caption{Autoregressive LMM inference with \ours{}}
\label{alg:lmm_v_reason}
\begin{algorithmic}[1]
\Require Pretrained LLM $f_\theta$; Encoder $\mathcal{E}$; video frames $\mathcal{V}$; text prompt $\mathcal{X}$; Sampler \Call{Sample};; maximum length $L_{\max}$; temperature $\tau$; vocabulary $\mathcal{W}$.
\Ensure Generated text $\hat{\mathbf{y}}$.

\Function{UpdateV}{$\mathsf{V}$}
  \State $\mathsf{V_L}' \gets \frac{\mathbb{V}_L + \Delta V}{||\mathbb{V}_L + \Delta V||}||\mathbb{V}_L||$ \Comment{add trainable offset and normalize}
  \State \Return $\mathsf{V_L}'$
\EndFunction

\Function{Optimize}{$\mathsf{\boldsymbol{\ell}_{N}, \Delta V},k$}
  \State $p_k \gets \Call{Softmax}{\boldsymbol{\ell}_{N}}$
  \State $H_k \gets -\sum_{i\in |\mathcal{W}|}p^i_k(\Delta V) \text{log}( p^i_k(\Delta V))$
  \State $\alpha_k = 
    \begin{cases}
        -1 \ \text{if} \ H_k^{ema} < H_{peak}^{ema},\\
        +1 \ \text{otherwise},\\
    \end{cases}$ \Comment{compute alpha}
  \State $\mathcal{L}_{switch}(p_k; \mathsf{\Delta V})\  \gets -\alpha_k H_k$ \Comment{compute loss} 
  \State $\mathsf{\Delta V} \gets \argmin\mathcal{L}_{switch}(p_k; \mathsf{\Delta V})$ \Comment{update parameters}
  \State \Return $\mathsf{\Delta V}$
\EndFunction

\State $\mathbf{z}_{1:N} \gets {\mathcal{E}}({\mathcal{V}, \mathcal{X}})$
\State $(\boldsymbol{\ell}_N, \mathsf{KV}) \gets f_\theta(\mathbf{z}_{1:N})$ \Comment{prefill: compute logits and full KV cache}
\State $\hat{y}_1 \gets \Call{Sample}{\boldsymbol{\ell}_N, \tau}$
\State $\hat{\mathbf{y}} \gets [\hat{y}_1]$
\State $t \gets 1$

\While{$t < L_{\max}$ and $\hat{y}_t \neq \text{[EOS]}$}
  \State $\mathsf{V} \gets \Call{UpdateV}{\mathsf{V}, \mathcal{I}_v, \pi}$
  \State $\mathsf{\Delta V} \gets \Call{Optimize}{\boldsymbol{\ell}_{N}, \mathsf{\Delta V},k}$
  \State $(\boldsymbol{\ell}_{N+t}, \mathsf{KV}) \gets f_\theta(\hat{y}_t \,|\, \mathsf{KV})$
  \State $\hat{y}_{t+1} \gets \Call{Sample}{\boldsymbol{\ell}_{N+t}, \tau}$
  \State $\hat{\mathbf{y}} \gets [\hat{\mathbf{y}}; \hat{y}_{t+1}]$
  \State $t \gets t + 1$
\EndWhile

\State \Return $\hat{\mathbf{y}}$
\end{algorithmic}
\end{algorithm}

\begin{algorithm}[h]
\caption{Autoregressive LMM inference with \ourslite{}}
\label{alg:lmm_v_reason_lite}
\begin{algorithmic}[1]
\Require Pretrained LLM $f_\theta$; Encoder $\mathcal{E}$; video frames $\mathcal{V}$; text prompt $\mathcal{X}$; Sampler \Call{Sample};; maximum length $L_{\max}$; temperature $\tau$; pruning policy $\pi$ (e.g., keep ratio $r$ by importance).
\Ensure Generated text $\hat{\mathbf{y}}$.

\Function{PruneKV}{$\mathsf{KV}$, $\mathcal{I}_v$, $\pi$}
  \State $\mathcal{S} \gets \text{Score}(\mathsf{KV}, \mathcal{I}_v)$ \Comment{low L2-norm}
  \State $\mathcal{K} \gets \text{Select}(\mathcal{I}_v, \mathcal{S}, \pi)$ \Comment{indices to keep among video positions}
  \State $\mathcal{M} \gets \{\text{all text positions}\} \cup \mathcal{K}$ \Comment{full keep-set}
  \State $\mathsf{KV}' \gets \text{IndexSelect}(\mathsf{KV}, \mathcal{M})$ \Comment{prune keys/values along sequence dimension}
  \State \Return $\mathsf{KV}'$
\EndFunction

\State $\mathbf{z}_{1:N} \gets {\mathcal{E}}({\mathcal{V}, \mathcal{X}})$
\State $(\boldsymbol{\ell}_N, \mathsf{KV}) \gets f_\theta(\mathbf{z}_{1:N})$ \Comment{prefill: compute logits and full KV cache}
\State $\mathcal{I}_v \gets \{1,\dots,N_v\}$ \Comment{positions of video tokens}
\State $\mathsf{KV} \gets \Call{PruneKV}{\mathsf{KV}, \mathcal{I}_v, \pi}$ \Comment{KV-cache pruning for efficiency}
\State $\hat{y}_1 \gets \Call{Sample}{\boldsymbol{\ell}_N, \tau}$
\State $\hat{\mathbf{y}} \gets [\hat{y}_1]$
\State $t \gets 1$

\State $\hat{\mathbf{y}} \gets \text{AutoRegressive}[\hat{\mathbf{y}}; \hat{y}_{1}]$ \Comment{inference optimization same as \textbf{Algorithm 1}}

\State \Return $\hat{\mathbf{y}}$
\end{algorithmic}
\end{algorithm}

\section{Inference Time and GPU Memory}\label{app:inftime}
Table \ref{tab:inference_time} presents the inference time, measured in seconds, of the baseline Qwen-2.5-VL-7B, \ours{}-7B, \ours{}-7B \texttt{(Lite)}, and Video-R1 across the six video reasoning benchmarks. All experiments were conducted on input videos with maximum video pixels set to $128\times28\times28$ and 32 frames temporal length. 
The reported results are the average over 50 samples.

From the results, it is evident that \ours{} and \ourslite{} consistently outperforms Video-R1 in terms of wall-clock inference time except for VideoMMMU. Specifically, \ours{} reduces inference time by approximately \textbf{20–67\%} compared to Video-R1 across the evaluated benchmarks. For instance, on TempCompass, the inference time decreases from 11.8 seconds per sample to 3.9 seconds per sample, while on MVBench, the reduction is from 10.7 seconds per sample to 4.1 seconds per sample. Fig. 1(d) shows that \ours{} has the maximum average output token count for VideoMMMU dataset and so using a step-size of 4 results in more number of optimization steps as compared to other datasets. This explains the anomaly observed in VideoMMMU results where the inference time is higher than Video-R1-7B.
Further, comparing \ours{} and \ourslite{} shows that token pruning introduces additional latency that increases the inference time marginally (+0.23 seconds) as compared to the full version without any pruning. These results highlight that \ours{}-7B  and \ours{}-7B \texttt{(Lite)} achieves a significant efficiency advantage in wall-clock inference time over the RL-trained model while narrowing the gap to within 0.6\% accuracy as demonstrated in Table 1.

We report the peak GPU memory usage for all models and compare \ourslite{} with \ours{} to show the benefit of our pruning variant in reducing GPU memory requirements. Table~\ref{tab:gpu_memory} shows that both \ours{} and \ourslite{} increase the memory overhead slightly compared to the baseline Qwen-2.5-VL-7B and the Video-R1-7B model as expected due to the additional memory overhead in optimization. Note that the memory overhead is much lower than optimizing for all decoder layers in the KV-cache. To further reduce the overhead, we introduced the lite variant \ourslite{}. The table shows that \ourslite{} reduces the average memory requirement across all datasets by \textbf{11.6\%} as compared to the full variant. In particular, the memory requirements drop by \textbf{20\%} on datasets with longer output token count length such as VideoMMMU (see Fig. 1(d)) suggesting the effectiveness of the proposed Lite variant. Notably, the peak GPU memory of \ourslite{} method is always lower than 32GB for the 7B model (on the datasets tested). This shows that the proposed lite variant is more suited for relatively smaller GPUs (e.g., 32GB V100 GPUs) and would not require more expensive GPUs like the Full variant.

\begin{table*}[h]
\centering
\scriptsize
\caption{Inference time (in seconds/sample) of Qwen-2.5-VL-7B, \ours{}-7B \texttt{(Lite)}, \ours{}-7B, and Video-R1-7B across different video reasoning benchmarks. Averaged over 50 samples from each dataset.}
\label{tab:inference_time}
\setlength{\tabcolsep}{1pt}
\begin{tabular}{l|c|c|c|c|c|c|c}
\toprule
\textbf{Model} & 
\textbf{VSI-Bench } &
\textbf{VideoMMMU } & 
\textbf{MMVU } & 
\textbf{TempCompass } & 
\textbf{MVBench } & 
\textbf{VideoMME } & \textbf{Average} \\
\midrule
Qwen-2.5-VL-7B & 3.80 & 9.02  & 6.73  &  2.86 & 3.30  & 4.37  & 5.01  \\
\midrule
Video-R1-7B & 10.17  &  \textbf{11.72} &  11.61 &  11.77  & 10.69  & 11.42  &  11.23 \\
\ours{}-7B \texttt{(Lite)} & 
\textbf{5.43 \textcolor{ForestGreen}{[$\downarrow$46.6\%]}} & 
14.18 \textcolor{Crimson}{[$\uparrow$21.0\%]} & 
\textbf{8.86 \textcolor{ForestGreen}{[$\downarrow$23.7\%]}} & 
\textbf{4.18 \textcolor{ForestGreen}{[$\downarrow$64.5\%]}} & 
\textbf{4.45 \textcolor{ForestGreen}{[$\downarrow$58.4\%]}} & 
\textbf{6.64 \textcolor{ForestGreen}{[$\downarrow$41.9\%]}} & 
\textbf{7.29 \textcolor{ForestGreen}{[$\downarrow$35.1\%]}} \\
\ours{}-7B & 
\textbf{5.06 \textcolor{ForestGreen}{[$\downarrow$50.2\%]}} & 
13.83 \textcolor{Crimson}{[$\uparrow$18.0\%]} & 
\textbf{9.28 \textcolor{ForestGreen}{[$\downarrow$20.0\%]}} & 
\textbf{3.87 \textcolor{ForestGreen}{[$\downarrow$67.1\%]}} & 
\textbf{4.13 \textcolor{ForestGreen}{[$\downarrow$61.4\%]}} & 
\textbf{6.18 \textcolor{ForestGreen}{[$\downarrow$45.9\%]}} & 
\textbf{7.06 \textcolor{ForestGreen}{[$\downarrow$37.1\%]}} \\
\bottomrule
\end{tabular}

\end{table*}

\begin{table*}[h]
\centering
\scriptsize
\caption{Peak GPU memory (in GB) of Qwen-2.5-VL-7B, \ours{}-7B \texttt{(Lite)}, \ours{}-7B, and Video-R1-7B across different video reasoning benchmarks. Averaged over 50 samples from each dataset.}
\label{tab:gpu_memory}
\setlength{\tabcolsep}{1pt}
\begin{tabular}{l|c|c|c|c|c|c|c}
\toprule
\textbf{Model} & 
\textbf{VSI-Bench } &
\textbf{VideoMMMU } & 
\textbf{MMVU } & 
\textbf{TempCompass } & 
\textbf{MVBench } & 
\textbf{VideoMME } & \textbf{Average} \\
\midrule
Qwen-2.5-VL-7B & 16.55 & 16.65  & 16.60  &  16.47 &  16.51  & 16.53  &  16.55 \\
Video-R1-7B & 16.70  & 16.74  &  16.73 & 16.68  &  16.70 & 16.69  & 16.71  \\
\midrule
\ours{}-7B & 23.95  &  38.48 &  29.91 & 22.32  & 23.28  & 25.56  &  27.25 \\

\ours{}-7B \texttt{(Lite)} & 
\textbf{22.41 \textcolor{ForestGreen}{[$\downarrow$6.4\%]}} & 
\textbf{30.79 \textcolor{ForestGreen}{[$\downarrow$20.0\%]}} & 
\textbf{25.05 \textcolor{ForestGreen}{[$\downarrow$16.2\%]}} & 
\textbf{21.45 \textcolor{ForestGreen}{[$\downarrow$3.9\%]}} & 
\textbf{21.86 \textcolor{ForestGreen}{[$\downarrow$6.1\%]}} & 
\textbf{22.89 \textcolor{ForestGreen}{[$\downarrow$10.5\%]}} & 
\textbf{24.08 \textcolor{ForestGreen}{[$\downarrow$11.6\%]}} \\

\bottomrule
\end{tabular}

\end{table*}

\begin{table*}[h]
\centering
\small
\caption{Comparison of Qwen-2.5-VL-7B, \ours{}, and \ours{}-7B \texttt{(Lite)} on VideoMME dataset. The differences with the baseline are denoted in red and green colors.}
\label{tab:videomme_dur}
\begin{tabular}{lc|c|c|c}
\toprule
Model & Mean Acc. & Short & Medium & Long \\
\midrule
Qwen-2.5-VL-7B & 53.7 & 64.6 & 50.4 & 46.1 \\
\midrule
\ours{}-7B \texttt{(Lite)} & 
53.5 \textcolor{Crimson}{[–0.2]} & 
\textbf{66.4 \textcolor{ForestGreen}{[+1.8]}} & 
49.7 \textcolor{Crimson}{[–0.8]} & 
44.3 \textcolor{Crimson}{[–1.8]} \\
\ours{}-7B & 
\textbf{54.9 \textcolor{ForestGreen}{[+1.2]}} & 
\textbf{66.4 \textcolor{ForestGreen}{[+1.8]}} & 
\textbf{51.2 \textcolor{ForestGreen}{[+0.8]}} & 
\textbf{47.0 \textcolor{ForestGreen}{[+0.9]}} \\
\bottomrule
\end{tabular}

\end{table*}

\begin{table*}[h]
\centering
\scriptsize
\caption{Ablation studies on pruning and learning rates for the variant using 128 px and 32 frames.}
\setlength{\tabcolsep}{2pt}
\resizebox{\textwidth}{!}{%
\begin{tabular}{l c c c c c c c c}
\toprule
\textbf{Model} & \textbf{VSI-Bench} & \textbf{VideoMMMU} & \textbf{MMVU} & \textbf{MVBench} & \textbf{TempCompass} & \textbf{VideoMME} & Avg & Avg\\
& (Acc/MRA) & & (mc) & & & (wo sub) &  & (wo mra) \\
\midrule
Qwen2.5-VL-7B &  28.1 (33.8/22.3) & 45.8 & 61.3 & 60.7 & 72.4 & 53.7 & 50.0& 54.6\\
\midrule
Qwen2.5-VL-7B + 50\% Pruning &  \textbf{28.2 (34.5/21.9)} \textbf{\textcolor{ForestGreen}{[+0.1]}}& \textbf{45.1} \textbf{\textcolor{Crimson}{[–0.7]}} & \textbf{61.3} \textbf{\textcolor{black}{[+0.0]}} & \textbf{60.0} \textbf{\textcolor{Crimson}{[–0.7]}} & \textbf{72.8} \textbf{\textcolor{ForestGreen}{[+0.4]}} & \textbf{52.8} \textbf{\textcolor{Crimson}{[–0.9]}} & \textbf{49.8} \textbf{\textcolor{Crimson}{[–0.2]}} & \textbf{54.4} \textbf{\textcolor{Crimson}{[–0.2]}}\\
\hline
\ours{}-7B \texttt{(Lite)}; (lr: 3e-4) &  \textbf{30.5 (37.3/23.7)} \textbf{\textcolor{ForestGreen}{[+2.4]}} & \textbf{46.7} \textbf{\textcolor{ForestGreen}{[+1.6]}} & \textbf{64.8} \textbf{\textcolor{ForestGreen}{[+3.7]}} & \textbf{60.6} \textbf{\textcolor{Crimson}{[–0.1]}} & \textbf{72.3} \textbf{\textcolor{Crimson}{[–0.1]}} & \textbf{53.5} \textbf{\textcolor{Crimson}{[–0.2]}} & \textbf{51.3} \textbf{\textcolor{ForestGreen}{[+1.3]}} & \textbf{55.9} \textbf{\textcolor{ForestGreen}{[+1.3]}}\\
\bottomrule
\label{tab:ablation_lr_pruning}
\end{tabular}
}
\end{table*}

\paragraph{Trainable memory computation example for the controller.}
Let us assume a fixed video token length of 1920 for analysis. Then the proposed controller introduces a parameter tensor of shape $(1,4,1920,128)$ for Qwen-2.5-VL-7B model, 
amounting to $N = 983{,}040$ trainable scalars. 
In FP32, this corresponds to $N \times 4 \,\text{bytes} = 3.84$~MiB of weights, 
while in FP16 the footprint is $1.92$~MiB. 
During training with the AdamW optimizer, additional memory is required for the gradient 
and two moment estimates of the same size as the parameters. 
Thus, in pure FP32 training the memory becomes 
$4 \times 3.75 = 15.36$~MiB (weights + gradients + $m$ + $v$). 
Since, the controller is used only as an additive bias (element-wise addition) the arithmetic cost is negligible ($\sim N$ adds, i.e.,\ $<10^{6}$ adds). The operation above is tiny compared to the bulk of transformer computation (attention and large dense projections), which typically entail orders of magnitude more FLOPs per token for typical hidden sizes and sequence lengths; therefore the controller’s compute overhead is minimal in most deployments. Note that the total GPU memory required for inference-time optimization will also include the memory required for storing the activations and gradients of the last decoder layer in the model as discussed above.

\section{Analysis on Video duration}\label{app:analysis}
We investigate the effect of video duration on the performance of \ours{} using the VideoMME dataset, which provides annotations for short, medium, and long videos. Specifically, short videos are less than two minutes in duration, medium videos range from 4 to 15 minutes, and long videos span 30 to 60 minutes. Table~\ref{tab:videomme_dur} presents a detailed breakdown of the results for both \ours{} and its Lite variant across these duration categories. The full \ours{} model consistently achieves notable gains, with a substantial improvement on short videos (\textbf{+1.8\%}) and notable gains on medium (\textbf{+0.8\%}) and long (\textbf{+0.9\%}) videos. The Lite variant of \ours{} also yields a significant improvement on short videos (\textbf{+1.8\%}), comparable to the full model, but its performance decreases for medium and long videos. We attribute this decline to pruning, which likely removes important temporal or contextual details, thereby reducing accuracy for longer content.

\begin{figure*}[t]
\centering
\begin{minipage}{0.48\textwidth}
\centering
\includegraphics[width=\textwidth]{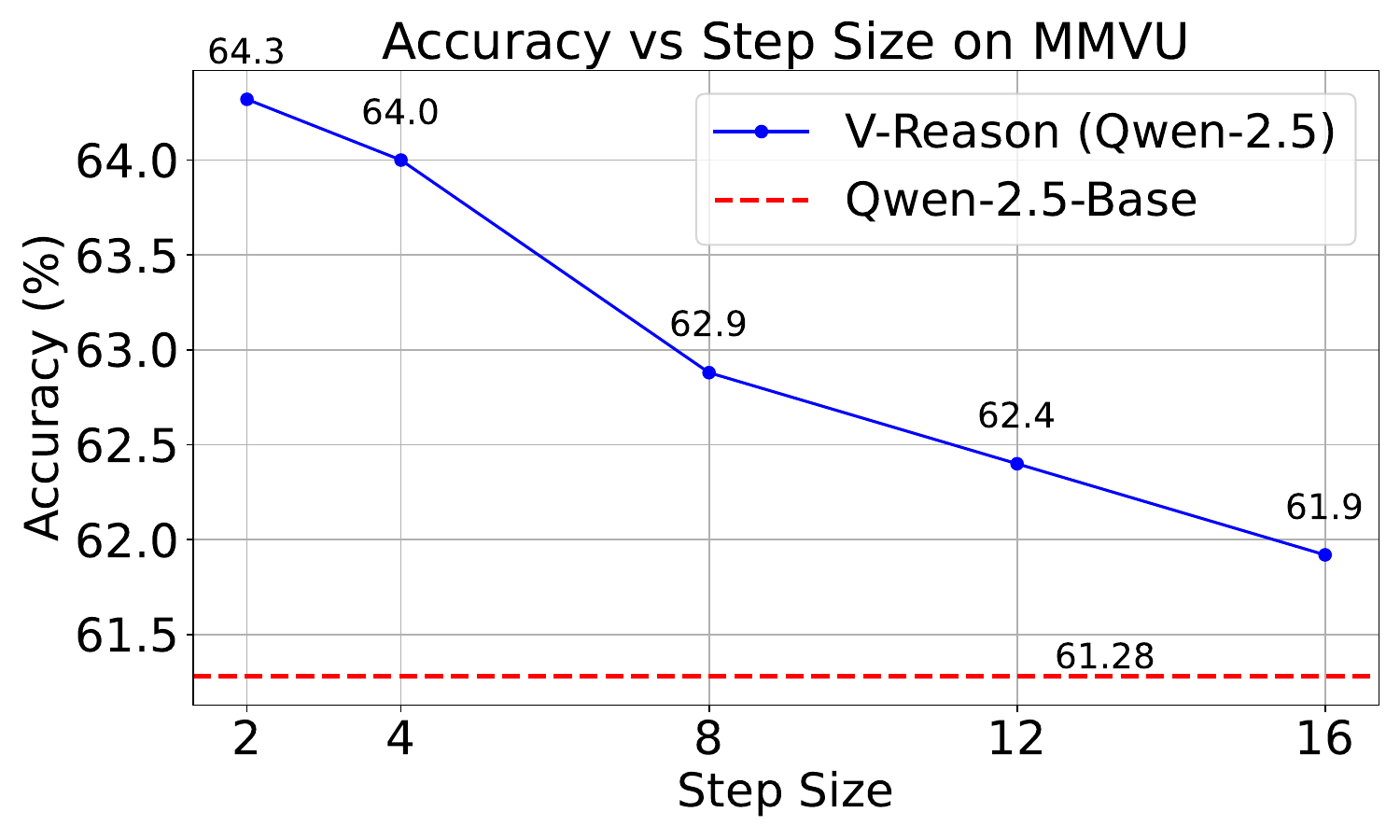}
\captionsetup{justification=centering}
\captionof{figure}{Optimization step-size ablations.}
\label{fig:mmvu_step_ablation}

\end{minipage}
\begin{minipage}{0.48\textwidth}
\centering
\includegraphics[width=\textwidth]{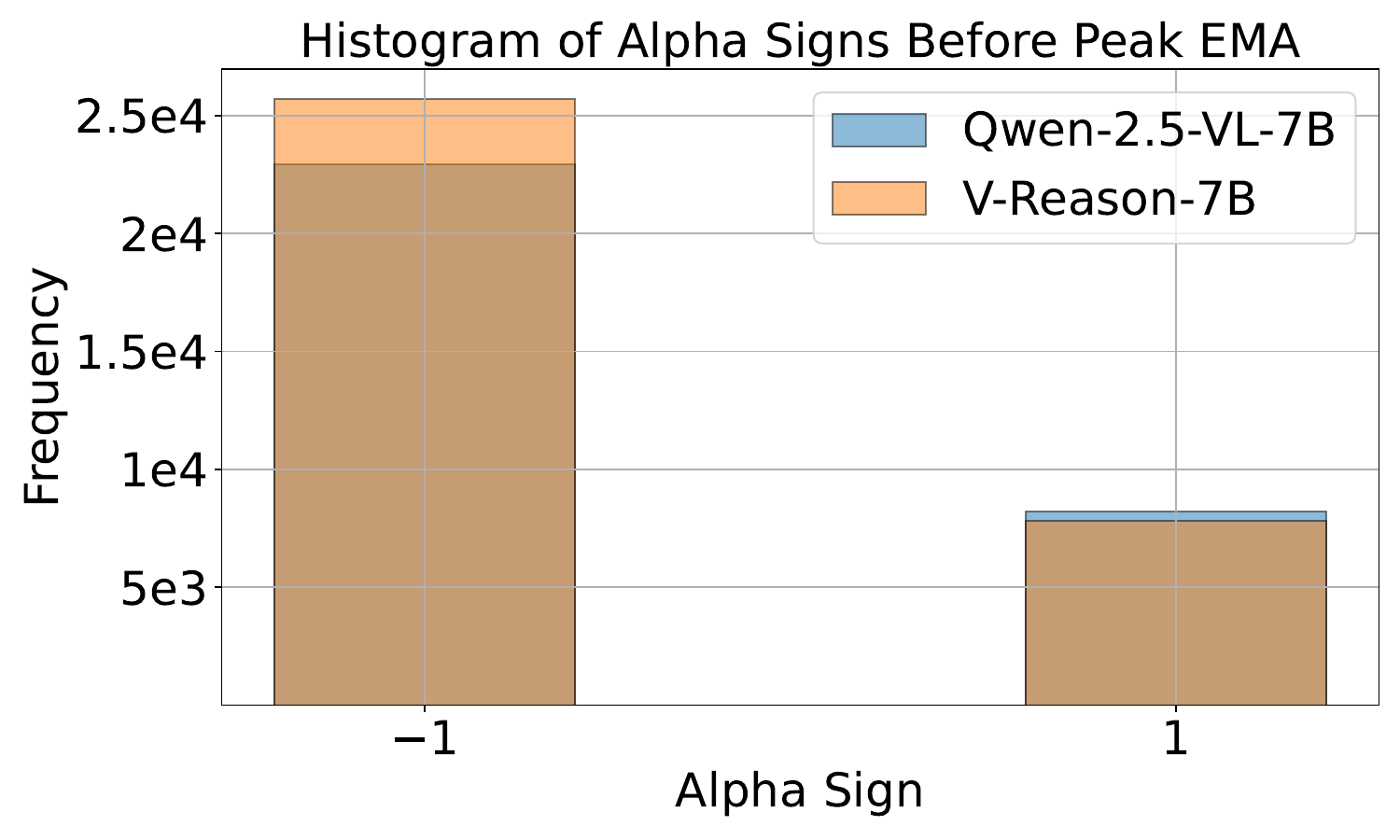}
\captionsetup{justification=centering}
\captionof{figure}{Alpha histogram before peak EMA entropy.}
\label{fig:alpha_histogram}
\end{minipage}
\label{fig:combined_results_row}
\end{figure*}

\begin{figure*}[h]
    \centering
    \begin{subfigure}[b]{0.48\linewidth}
        \centering
        \includegraphics[width=\linewidth]{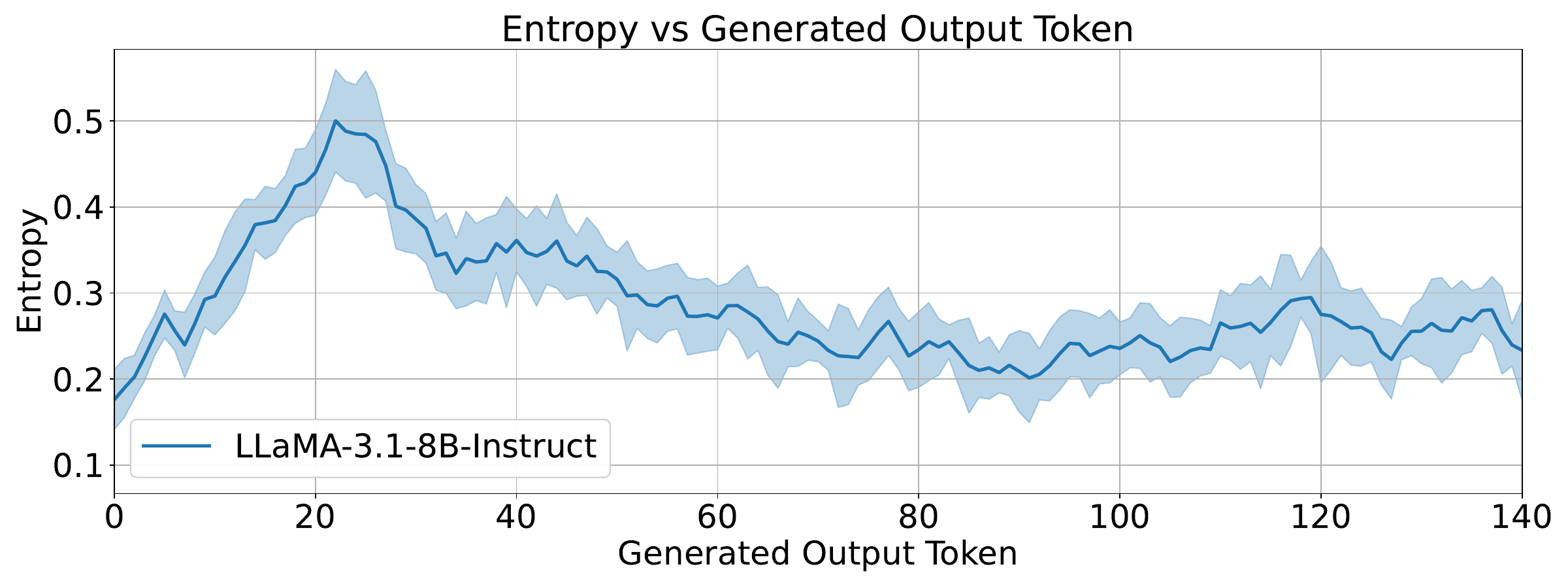}
        \caption{Llama 3.1}
        \label{fig:llama3.1}
    \end{subfigure}
    \hfill
    \begin{subfigure}[b]{0.48\linewidth}
        \centering
        \includegraphics[width=\linewidth]{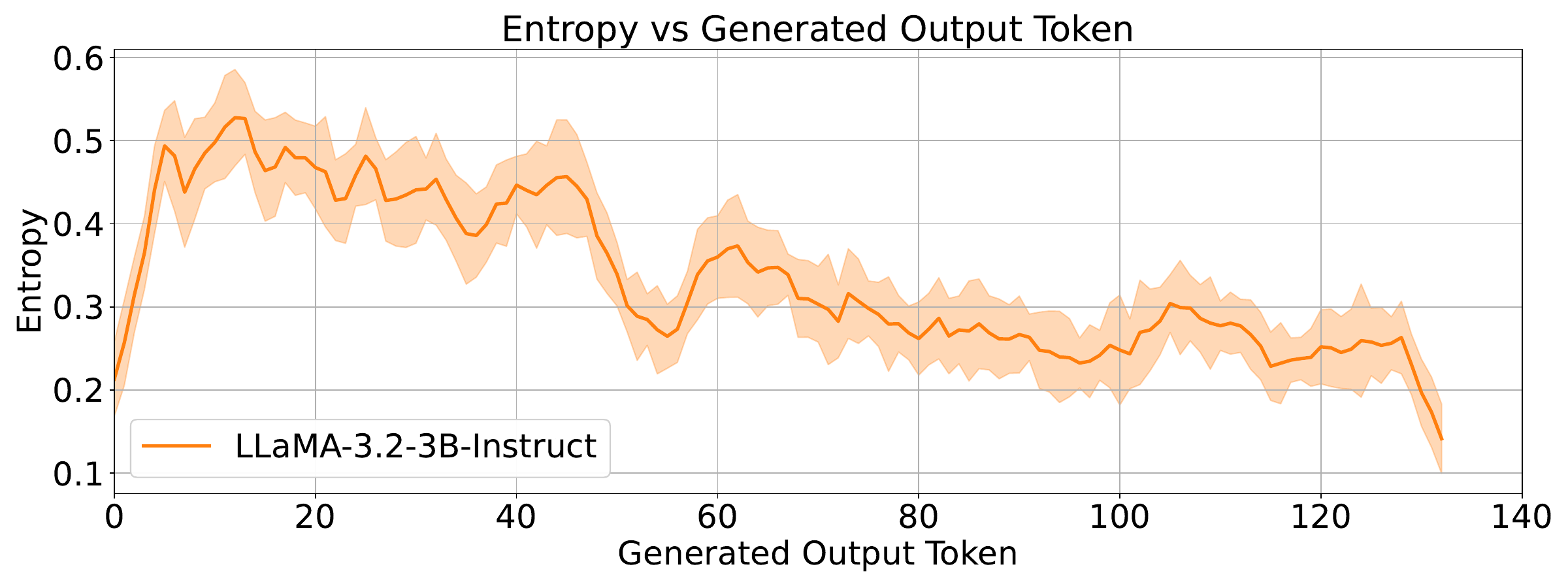}
        \caption{Llama 3.2}
        \label{fig:llama3.2}
    \end{subfigure}
    \caption{\textit{Llama analysis}: Entropy of the output distribution averaged over 100 samples of the MATH dataset~\citep{qwenmath}. Similar to Qwen LMMs, we see clear macro-exploration and macro-exploitation phases (having micro-exploration and micro-exploitations). Note that the entropies cannot be directly compared due to different training recipes and tokenizers across the different models but we do see that the larger model (more accurate) has a delayed entropy peak than the smaller model consistent with the trends seen in Qwen model series.}
    \label{fig:llamaentropy}
\end{figure*}

\begin{table*}[ht]
\centering
\scriptsize
\caption{Image Understanding benchmarks (Qwen2.5-VL,\ours{})}
\resizebox{0.8\textwidth}{!}{%
\begin{tabular}{lccccccccccc}
\toprule
\textbf{Method} & \textbf{GQA} & \textbf{TextVQA} & \textbf{SQA} & \textbf{SEED} & \textbf{MMVet} & \textbf{Avg.} \\
\midrule
Qwen2.5-VL-7B & 48.9 & 66.4 & \textbf{42.5} & \textbf{74.1} & 45.8 & 55.5 \\
\ours{}-7B \texttt{(Lite)} & \textbf{49.2} \textcolor{ForestGreen}{[+0.2]} & \textbf{67.7} \textcolor{ForestGreen}{[+1.3]} & \textbf{42.8} & \textbf{74.1} & \textbf{46.9} \textcolor{ForestGreen}{[+1.1]} & \textbf{56.1}\textcolor{ForestGreen}{[+0.6]}\\
\bottomrule

\label{tab:image_comparison}
\end{tabular}
}
\end{table*}

\begin{table*}[t]
\centering
\scriptsize
\caption{Comparison of performance of \ours{} on CoT-SFT/RL trained models.}
\setlength{\tabcolsep}{2pt}
\resizebox{0.8\textwidth}{!}{%
\begin{tabular}{l c c c c c c c c}
\toprule
\textbf{Model} & \textbf{VSI-Bench} & \textbf{V.MMMU} & \textbf{MMVU} & \textbf{MVBench} & \textbf{TCompass} & \textbf{V.MME} & Avg & Avg wo mra\\
\midrule
COT-SFT-7B  & 32.5 (\textbf{32.2}/32.8) & 43.7 & 59.2 & 60.9 & 69.5 &  54.9 & 50.4& 53.4\\ %
\quad+\ours{}  & \textbf{32.6} (32.1/\textbf{33.1}) & \textbf{45.1}\textcolor{ForestGreen}{[+1.4]} & \textbf{61.8}\textcolor{ForestGreen}{[+2.6]} & \textbf{62.0}\textcolor{ForestGreen}{[+1.1]} & \textbf{69.6}\textcolor{ForestGreen}{[+0.1]} &  \textbf{55.1}\textcolor{ForestGreen}{[+0.2]} & \textbf{51.3}\textcolor{ForestGreen}{[+0.9]}& \textbf{54.3}\textcolor{ForestGreen}{[+0.9]}\\ %
\midrule
Video-R1-7B  & \textbf{35.6} (\textbf{30.9}/39.2) & \textbf{48.8} & 64.0 & 64.1 & \textbf{73.3} &  58.7 & \textbf{54.1}& 56.6\\ %
\quad+\ours{}  & 34.2 (29.1/\textbf{39.3}) & \textbf{48.8} & \textbf{64.2}\textcolor{ForestGreen}{[+0.2]} & \textbf{64.8}\textcolor{ForestGreen}{[+0.7]} & \textbf{73.3} &  \textbf{59.4}\textcolor{ForestGreen}{[+0.7]} & \textbf{54.1}& \textbf{56.7}\textcolor{ForestGreen}{[+0.1]}\\ %
\bottomrule
\label{tab:video_comparison}
\end{tabular}
}
\end{table*}

\begin{figure}[h]
\scriptsize
\centering
\includegraphics[keepaspectratio, width=0.8\columnwidth]{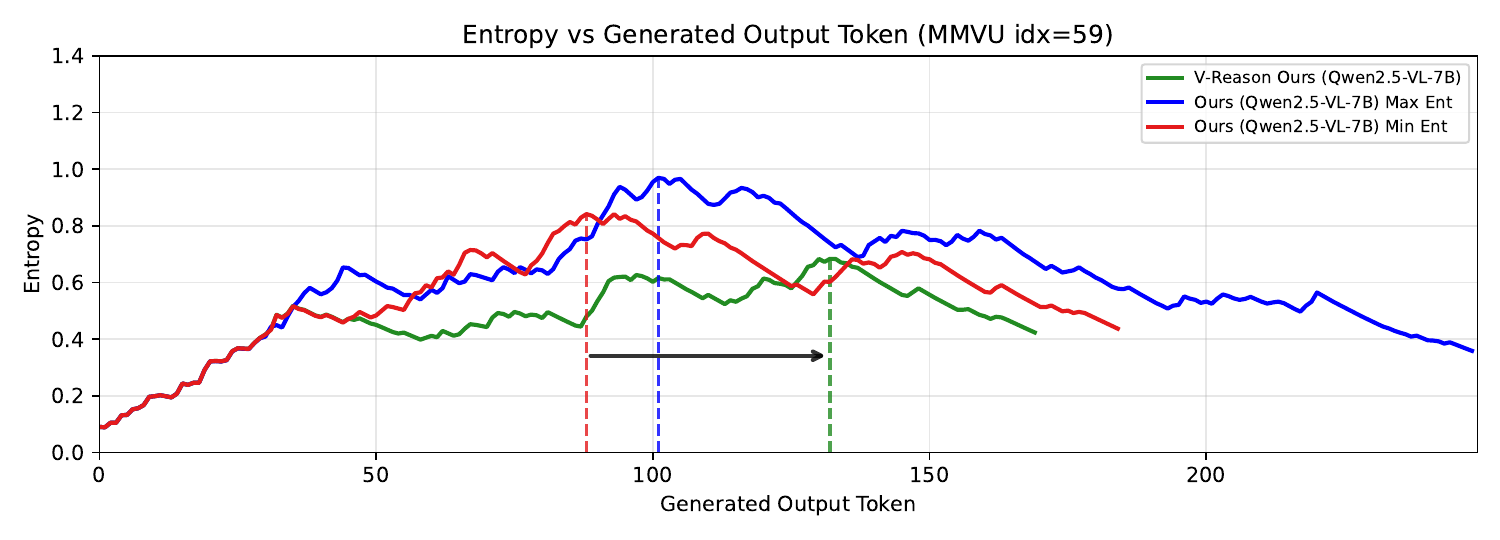}
\caption{Comparison of different optimization objectives. The arrow denotes the shift in the peak, i.e., longer exploration for \ours{}.}
\label{fig:qual}
\end{figure}

\section{Image Understanding Results}

We present image understanding results in Table~\ref{tab:image_comparison}. Our method consistently improves over Qwen2.5-VL-7B, achieving an average gain of \textbf{+0.6\%}, with improvements exceeding \textbf{+1\%} on TextVQA and MMVet. Although these gains are modest, this is expected given the non-reasoning centric nature of these benchmarks and the strong performance of the baseline model (highly saturated).

From the perspective of our framework, these results suggest that the benefits of inference-time optimization are not limited to video reasoning, but extend to general visual reasoning tasks as well. In particular, the improvements indicate that better control of the model’s exploration–exploitation behavior at inference can elicit more reliable reasoning even when the underlying knowledge is already present. Importantly, these gains are achieved without any RL or additional training, reinforcing the effectiveness of our training-free approach as a lightweight and broadly applicable enhancement to existing multimodal models.

\section{Results on SFT/RL trained models}

Our objective is \textbf{not} to replace reinforcement learning (RL), but to introduce a \emph{training-free, inference-time optimization framework} that improves reasoning efficiency and accuracy by better exploiting the pretrained model’s latent capabilities. This perspective is consistent with the view of reasoning as a search process over existing knowledge, where improved control at inference can elicit stronger reasoning behavior without additional supervision. Importantly, we do not claim to fundamentally expand the model’s reasoning capacity, but rather to surface and refine it through more effective exploration–exploitation dynamics.

As shown in Table~\ref{tab:video_comparison}, \ours{} consistently improves both CoT-SFT and RL-trained models. The gains are naturally smaller for RL models, which already exhibit more optimal reasoning behavior, but remain consistent in the range of \textbf{+0.2-0.7\%}. This aligns with our hypothesis that RL primarily improves the search process, which our method approximates at inference time through controlled optimization. In contrast, larger gains are observed for CoT-SFT models (up to \textbf{+0.9\%}), where reasoning behavior is less optimized and thus benefits more from improved inference-time control.

Notably, these improvements are achieved without any additional training, supervision, or modification of model parameters, highlighting the effectiveness of entropy-guided inference-time adaptation. Furthermore, our method remains complementary to existing training paradigms, suggesting that it can be seamlessly integrated with both SFT and RL-based approaches. We also highlight the substantial \textbf{+5.1\%} gain on Qwen-3-VL-8B (Table 2 main paper), demonstrating that inference-time optimization can yield significant improvements even for strong pretrained models.

\section{Ablation Studies}\label{app:ablations}

In this section, we present additional ablation studies to assess the impact of hyperparameters used during inference-time optimization, including optimization step-size (update frequency), and we further analyze the frequency of alpha values before entropy maxima.

\paragraph{Pruning-Only.}
Table~\ref{tab:ablation_lr_pruning} compares \ours{} to a baseline model that implements pruning only. This shows that it is effective in maintaining the original performance with only \textbf{-0.2\%} decrease on average across all datasets. Surprisingly, it also has small gains over the baseline on the VSI-Bench and TempCompass datasets. When \ours{} is combined with pruning, the average gain (without MRA) increases from $-0.2$ to $\textbf{1.3}$. This shows that the reasoning gains derive mostly from the inference optimization. 
\paragraph{Learning rate.} Table~\ref{tab:ablation_lr_pruning} also reports results for \ours{} with a fixed learning rate of \texttt{3e-4} across six datasets. The method maintains the average performance reported in Table 1 under this setting with similar gains observed on VSI-Bench, VideoMMMU, and MMVU datasets and only a negligible drop in the performance on MVBench, Tempcompass, and VideoMME datasets, highlighting its robustness to variations in optimization hyperparameters.

\paragraph{Optimization Step-size.}
Figure~\ref{fig:mmvu_step_ablation} shows an ablation on optimization step-size on MMVU dataset. It shows that accuracy increases with decreasing step-size. Since  smaller step-sizes correspond to more optimization steps, there is a trade-off between  efficiency and accuracy (fewer steps lead to faster inference). Notably, \ours{} outperforms the base model for all step-sizes, demonstrating that even a few optimization steps can guide the model towards improved reasoning paths. 

\paragraph{Optimization Objective.} Table 3 in the paper showed that the proposed entropy-switching objective is better than alternatives such as min-entropy or max-entropy losses. Fig.~\ref{fig:qual} now shows how the alternative optimization objectives fail to reproduce the characteristic entropy dynamics (later peak, lower final entropy) associated with better reasoning, supporting the hypothesis that structured entropy regulation is key.

\paragraph{Alpha Switching.}

Figure~\ref{fig:alpha_histogram} shows the histogram of alpha values before the EMA peak is attained. \ours{} sacrifices a few micro-exploration steps ($\alpha=1$) for a substantially larger number of micro-exploitation steps ($\alpha=-1$), suggesting that it pursues more alternative paths during macro-exploration. This lengthens the macro exploration stage and delays the overall entropy peak.

\section{Entropy Curves with Other Models}\label{sec:othermodels}

\paragraph{Entropy Analysis of Llama Models.}
We extend our analysis to the Llama family in Figure~\ref{fig:llamaentropy}, examining Llama 3.1-8B-Instruct and Llama 3.2-3B-Instruct. It shows the entropy curves averaged over a subset of 100 samples on the MATH dataset. The observed entropy trajectories demonstrate macro-exploration and macro-exploitation phases that include the micro-exploration and micro-exploitation cycles as described in the paper. Figure~\ref{fig:llamaentropy} shows that our insights are not restricted for Qwen-based architectures. Consistent with our prior observations, the larger, higher-performing model exhibits a delayed entropy maximum compared to the smaller model, suggesting that our approach generalizes across diverse LLM architectures. This consistency across model families provides evidence that our entropy-based characterization extends beyond a single architectural lineage. We plot the models separately as the absolute entropy values are not directly comparable owing to differences in pretraining procedures and tokenization schemes.

\paragraph{Entropy Analysis of Phi Models.}
Continuing our cross-architecture analysis, we examine Phi 3-mini-4k-Instruct in Figure~\ref{fig:phientropy} by averaging over a subset of 100 samples on the MATH dataset. The generated output token entropy behavior follows the same trend observed in both Qwen and Llama families with the macro-exploration/exploitation structure, providing additional evidence for the generality of these patterns.

\begin{figure*}[h]
    \centering
    \begin{subfigure}[b]{0.48\linewidth}
        \centering
        \includegraphics[width=\linewidth]{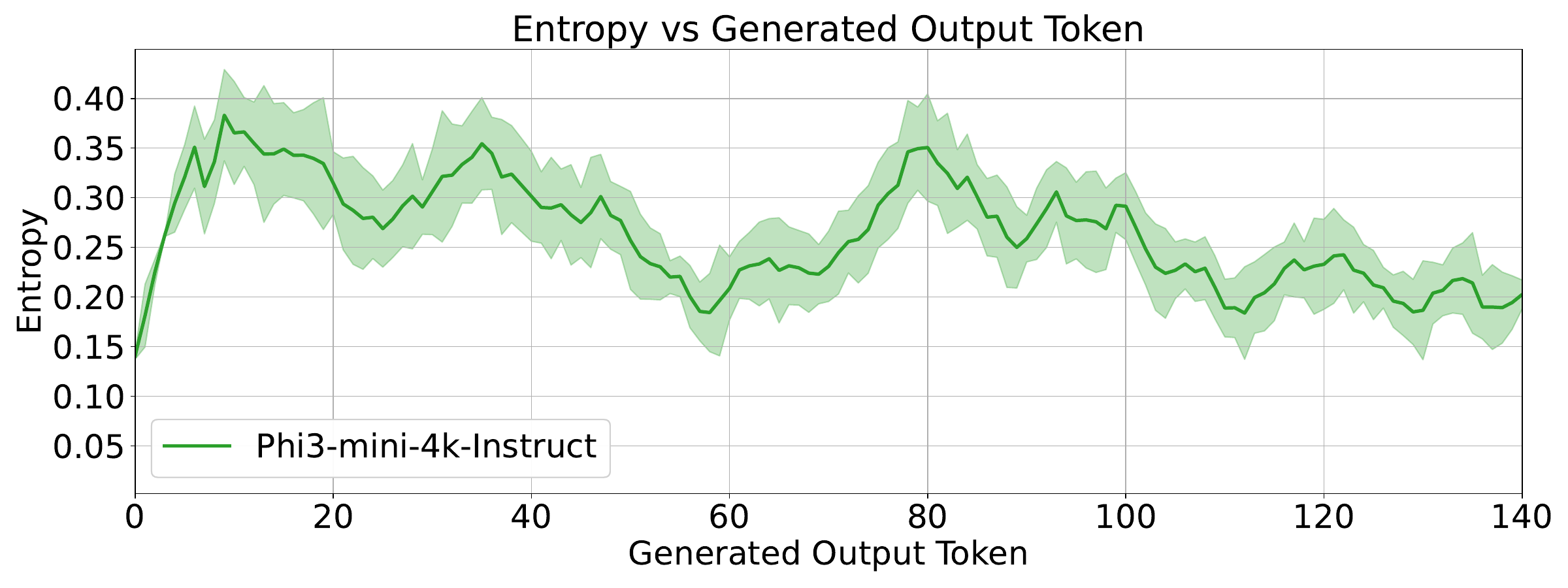}
        \caption{Phi 3}
        \label{fig:phi3}
    \end{subfigure}
    \hfill
    \begin{subfigure}[b]{0.48\linewidth}
        \centering
        \includegraphics[width=\linewidth]{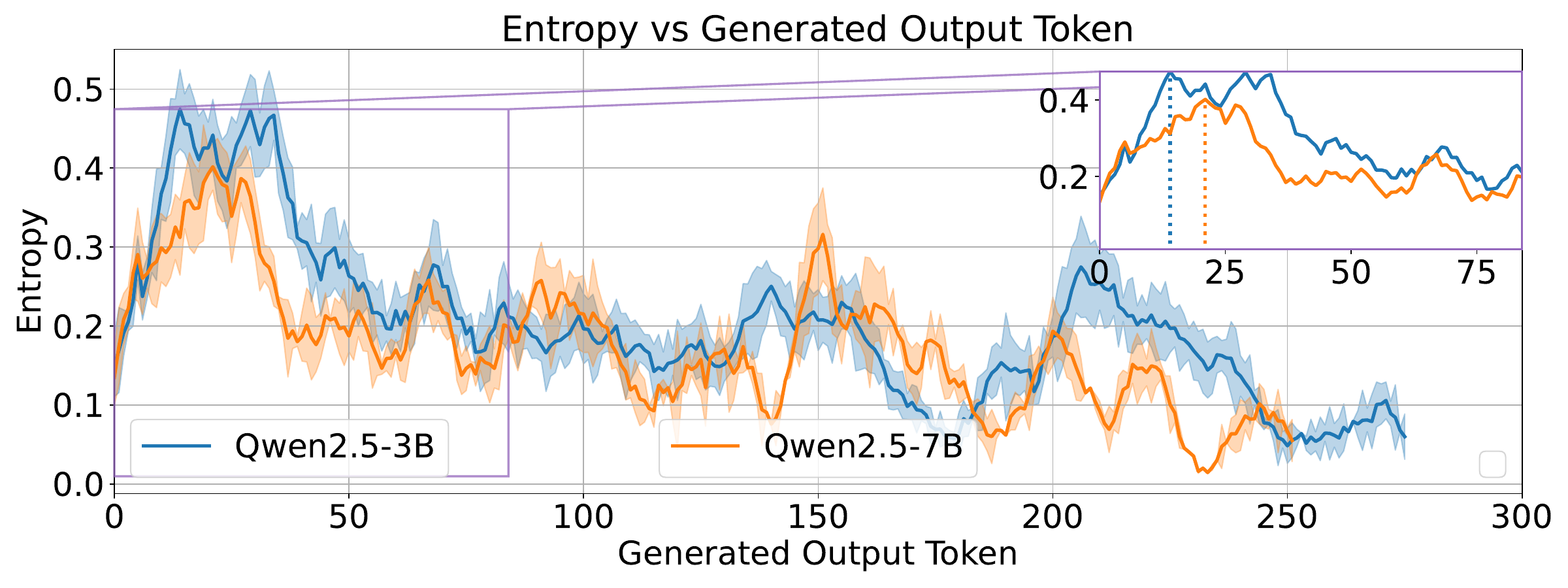}
        \caption{Qwen 2.5}
        \label{fig:qwen25}
    \end{subfigure}

    \caption{\textit{Phi and Qwen analysis}: Entropy of the output distribution averaged over 100 samples of the MATH dataset~\citep{qwenmath}. (a) Similar to Qwen LMMs, we observe clear macro-exploration and macro-exploitation phases (with micro-exploration and micro-exploitation) with the Phi model. (b) Qwen LLM family also shows the similar pattern observed in the LMMs with the better model showing delayed entropy maxima and lower entropy peak.}
    \label{fig:phientropy}
    \label{fig:llmentropyIntro}
\end{figure*}

\paragraph{Entropy Analysis of Qwen LLMs.}
Figure~\ref{fig:llmentropyIntro} shows the entropy curves of Qwen2.5 based LLMs averaged over a subset of 100 samples on the MATH dataset. It shows that the larger 7B model (more accurate) shows a delayed peak and lower entropy overall. These trends are consistent to those observed in the video models suggesting that the proposed approach can potentially be extended to LLMs. We leave this for future work since direct application of our value-cache controller is not feasible with text-only inputs as they lack the distinct modality embeddings (video tokens) that provide optimization targets in multimodal models. Addressing this requires non-trivial contributions which is beyond the scope of the current work. However, it is very exciting to see similar macro-exploration and macro-exploitation trends with cycles of micro-exploration and micro-exploitation that cause a delayed entropy maximum for the better model across architectures and even for LLMs.

\section{Limitations}\label{sec:limitations}

Although \ours{} demonstrates consistent improvements across benchmarks, there are certain limitations.
First, our approach relies on the knowledge of the pretrained model to explore alternative paths during the thinking process and so for certain tasks that are under-represented in the pretrained model, \ours{} can only provide modest gains compared to training-based approaches. 

Second, the Lite variant, while improving memory efficiency, incurs a measurable drop in accuracy for medium and long-duration videos, suggesting that pruning may discard valuable temporal information for those cases. Such limitations  can be investigated as future work, as described next.

\section{Future Work}\label{app:future_works}
To our knowledge, \ours{} is the first work that targets the \textit{video reasoning without training} problem. Hence, a number of exciting avenues exist for future research. 

First, our entropy-based objective is applied only at inference time; integrating it into model training could potentially yield stronger gains and is an avenue for potential future research. Other directions of future research include exploring alternative inference-time metrics and loss functions that can further enhance reasoning. 

Second, as a training-free framework, our method does not leverage task-specific supervision, which may limit its ability to capture nuanced reasoning strategies compared to reinforcement learning-based approaches. Hence, a combination of supervised finetuning and inference-time optimization-based reasoning techniques can also be explored in the future. Additionally, tailored solutions that can handle longer videos for the Lite variant can also be investigated.

Finally, although our proposed approach is motivated for videos, the idea of entropy-based inference-time optimization for enhanced reasoning is generic and can be extended to large language models (LLMs). We conducted a preliminary analysis of the entropy behavior of language models for MATH reasoning tasks and observed similar trends as the video models as discussed in section~\ref{sec:othermodels}.

\section{Additional Qualitative results}\label{sec:qualitative_appendix}

Figures \ref{fig:qual_3a}, \ref{fig:qual_3b}, \ref{fig:qual_3d}, \ref{fig:qual_3e}, \ref{fig:qual_3j}, and
\ref{fig:qual_3k} shows additional examples where the baseline Qwen-2.5-VL-7B failed to arrive at the correct solution while \ours{} arrived at the correct answer following an alternative reasoning path. The examples span a variety of topics in science, engineering, action recognition and counting. In particular, Figure~\ref{fig:qual_3j} shows an example where the base model becomes distracted by irrelevant context while \ours{} correctly interprets the frames. Similarly, for the counting task in Figure~\ref{fig:qual_3k}, the base model provides confident yet incorrect totals, while \ours{} avoids these mistakes and arrives at the correct answer through effective comparison.

 Figures \ref{fig:qual_3f}, \ref{fig:qual_3g}, \ref{fig:qual_3h}, and \ref{fig:qual_3i} show examples where both the baseline Qwen-2.5-VL-7B and \ours{} arrive at the correct answer while going through similar or alternative reasoning traces. In all these examples, \ours{} shows the consistent trend of longer exploration (delayed peak) and lower overall entropy induced by the micro-exploration and micro-exploitation cycles in our proposed optimization objective. In particular, Figure~\ref{fig:qual_3f} shows that \ours{} arrives at the correct solution using fewer output tokens as compared to the baseline as illustrated in Figure 1(d). Interestingly, in Figure~\ref{fig:qual_3h} \ours{} uses more output tokens to provide the correct answer as compared to the baseline. However, \ours{} results in a more confident answer as seen from the lower overall entropy as compared to the baseline. This trend of higher confidence and lower overall entropy is seen in all the examples of \ours{} suggesting the effectiveness of our macro-exploitation phase induced by our proposed objective function.

\begin{figure*}[h]
\centering
\includegraphics[keepaspectratio, width=\textwidth]{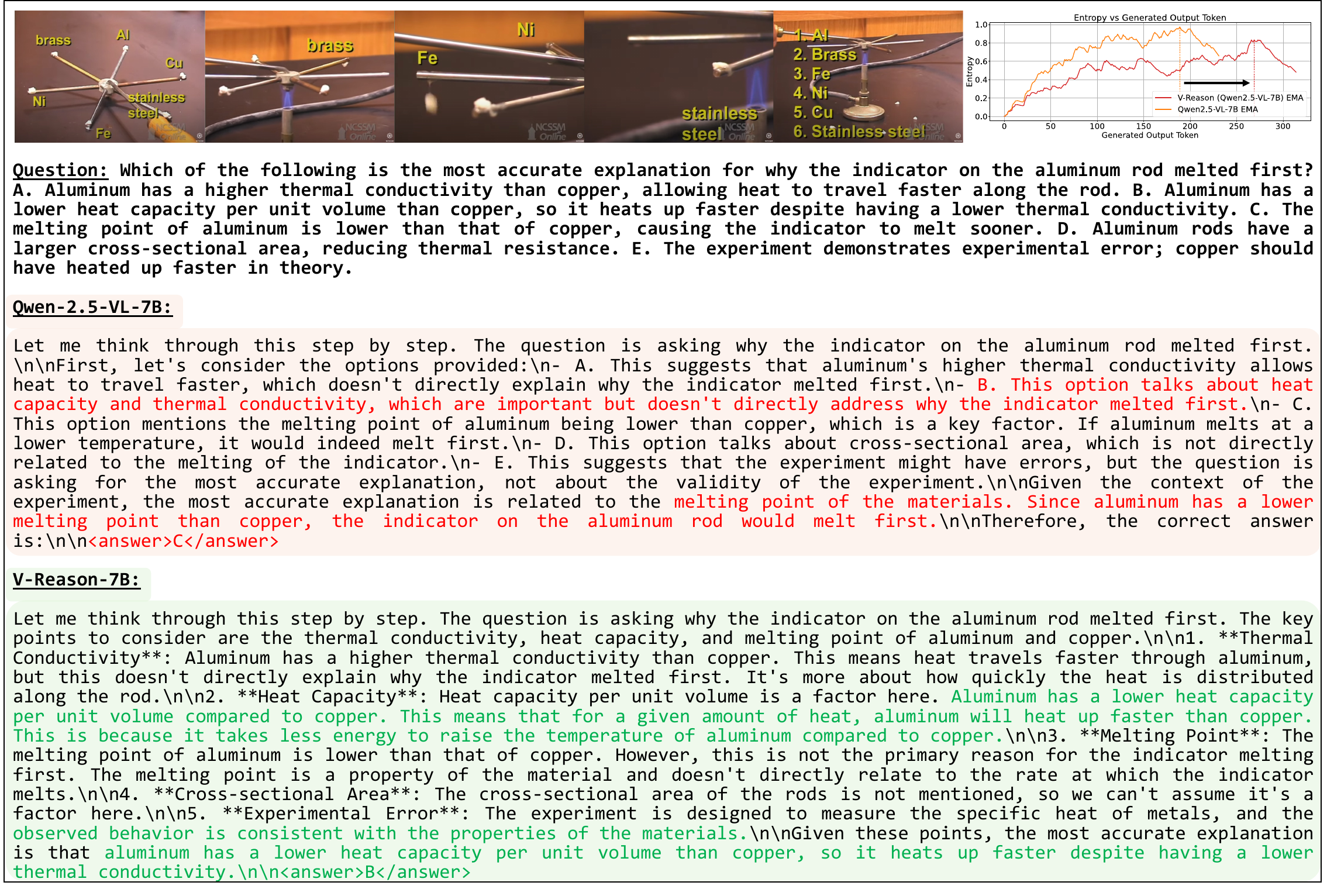}
\caption{Qualitative result: An example output and comparison with the baseline Qwen-2.5-VL-7B. Black arrow denotes the shift in the peak demonstrating longer exploration for \ours{}-7B.}
\label{fig:qual_3a}
\end{figure*}

\begin{figure*}[h]
\centering
\includegraphics[keepaspectratio, width=\textwidth]{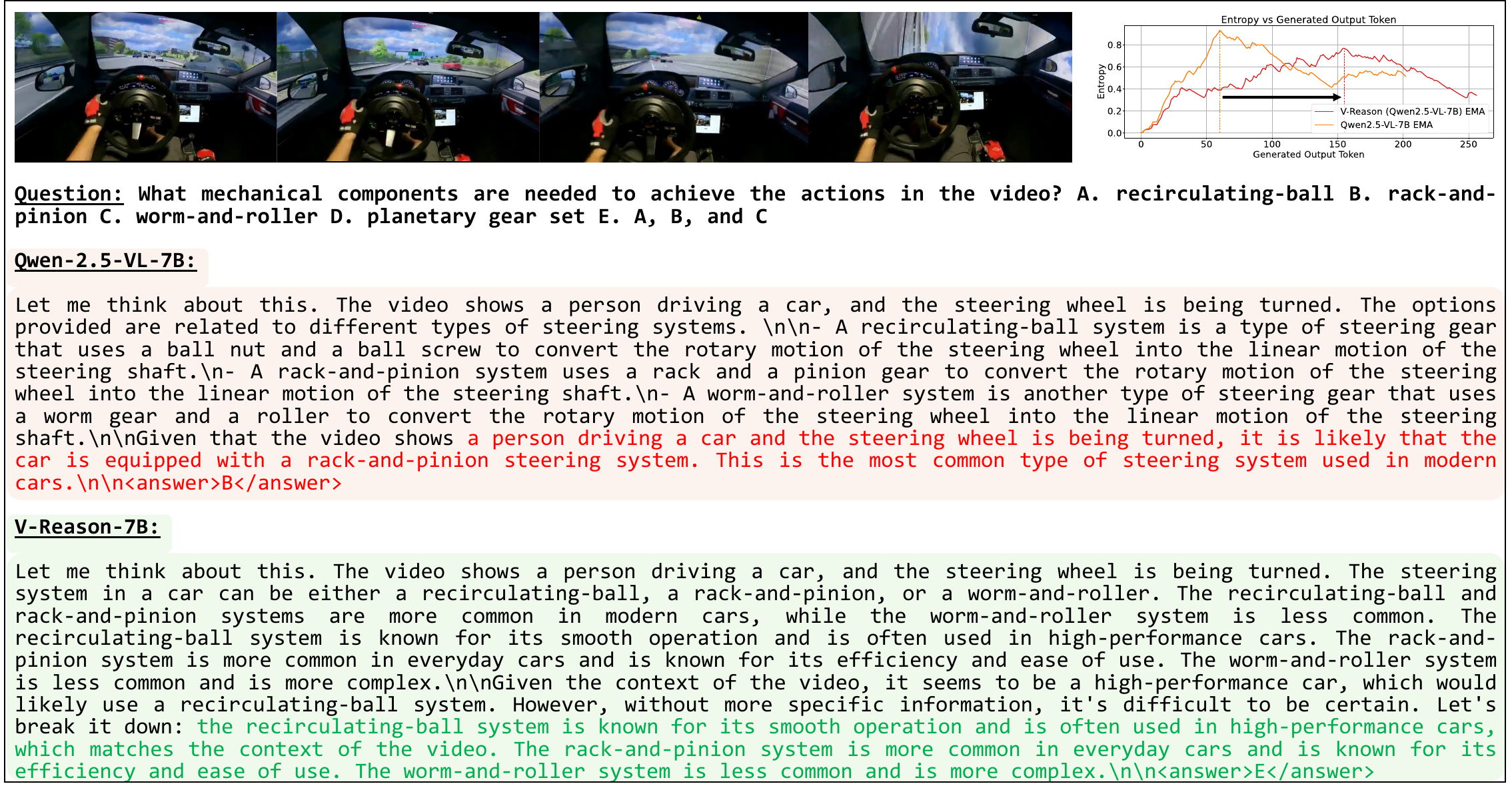}
\caption{Qualitative result: An example output and comparison with the baseline Qwen-2.5-VL-7B. Black arrow denotes the shift in the peak demonstrating longer exploration for \ours{}-7B.}
\label{fig:qual_3b}
\end{figure*}

\begin{figure*}[h]
\centering
\includegraphics[keepaspectratio, width=\textwidth]{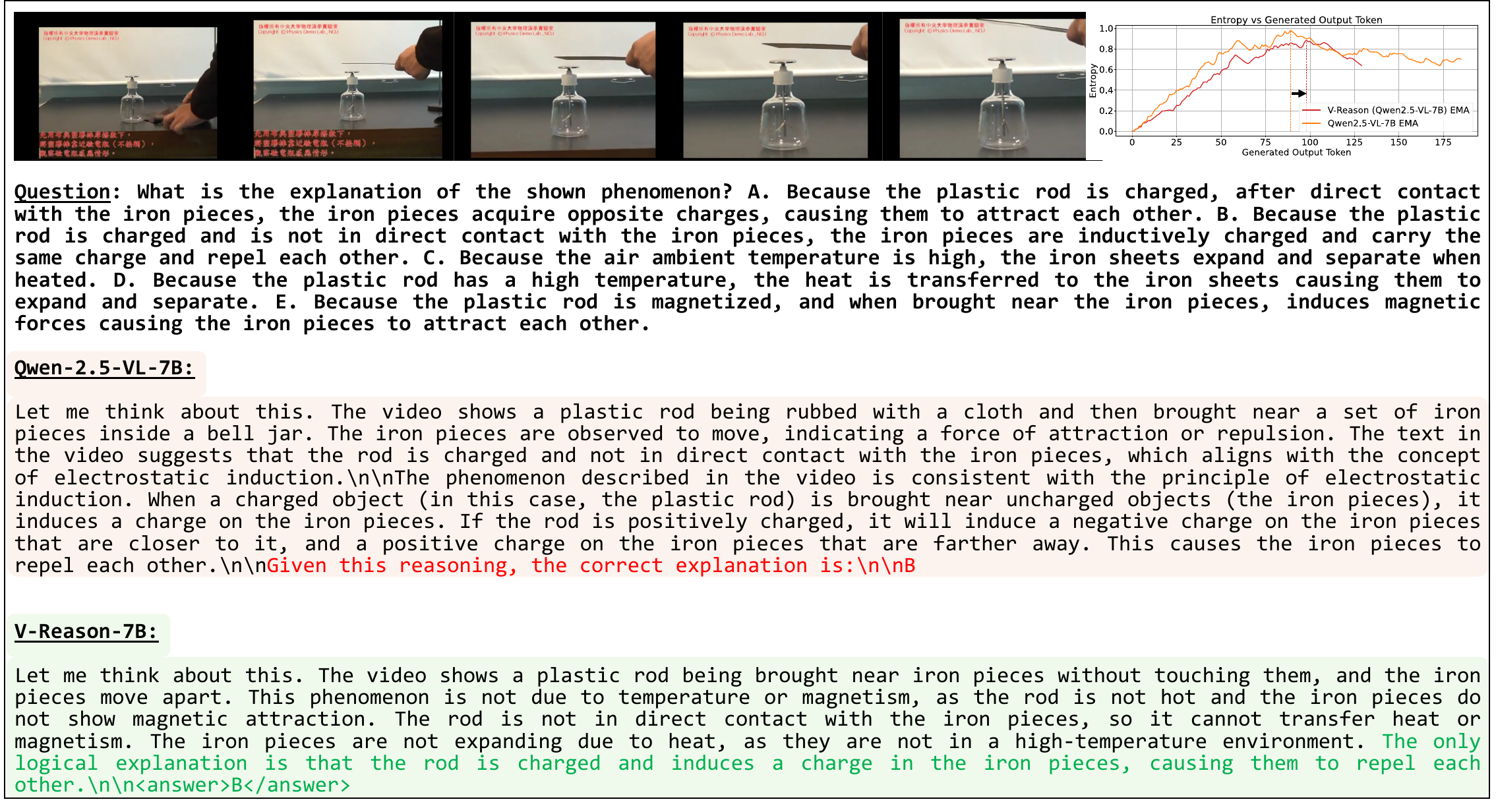}
\caption{Qualitative result: An example output and comparison with the baseline Qwen-2.5-VL-7B. Black arrow denotes the shift in the peak demonstrating longer exploration for \ours{}-7B.}
\label{fig:qual_3d}
\end{figure*}

\begin{figure*}[h]
\centering
\includegraphics[keepaspectratio, width=\textwidth]{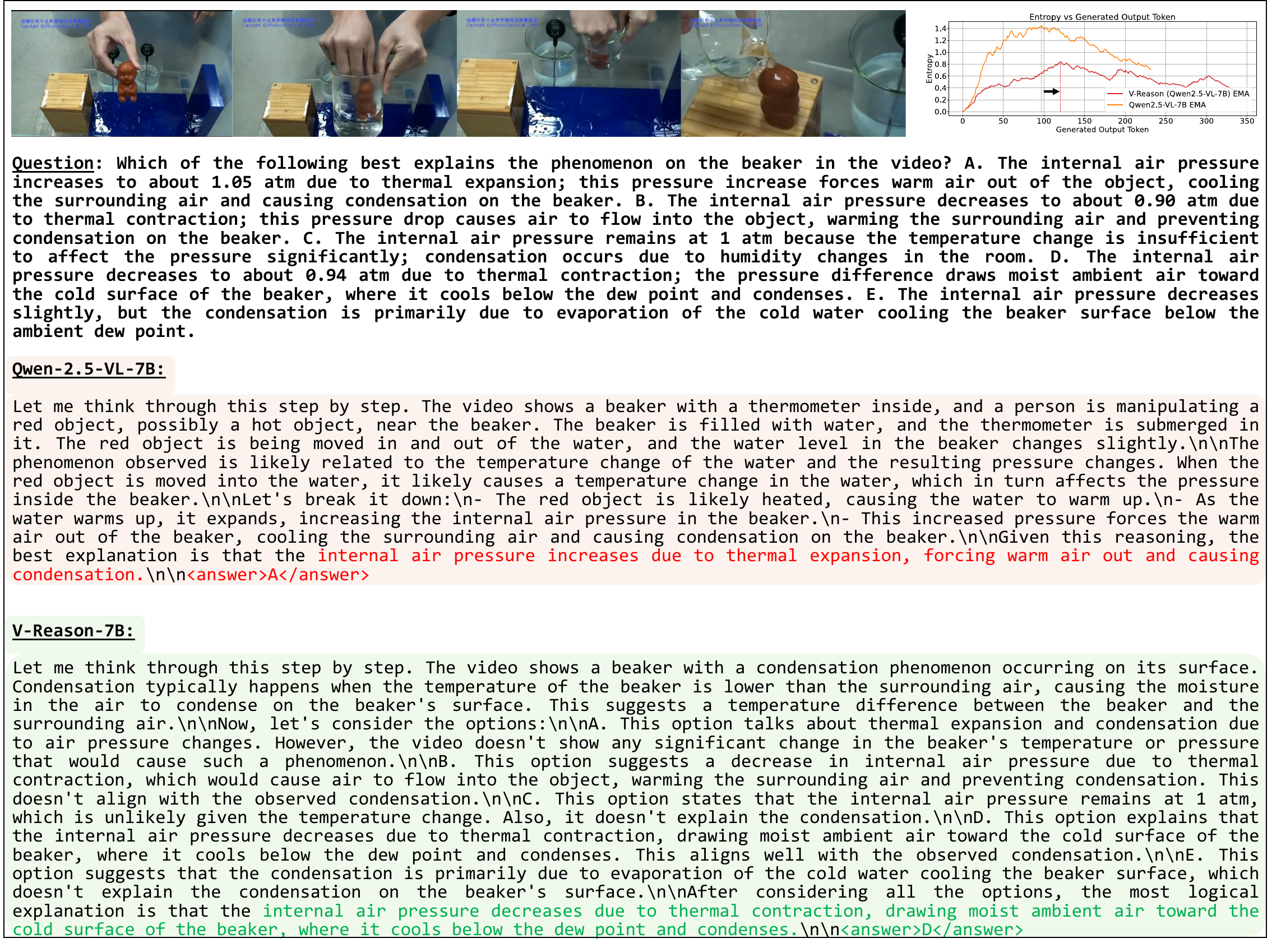}
\caption{Qualitative result: An example output and comparison with the baseline Qwen-2.5-VL-7B. Black arrow denotes the shift in the peak demonstrating longer exploration for \ours{}-7B.}
\label{fig:qual_3e}
\end{figure*}

\begin{figure*}[h]
\centering
\includegraphics[keepaspectratio, width=\textwidth]{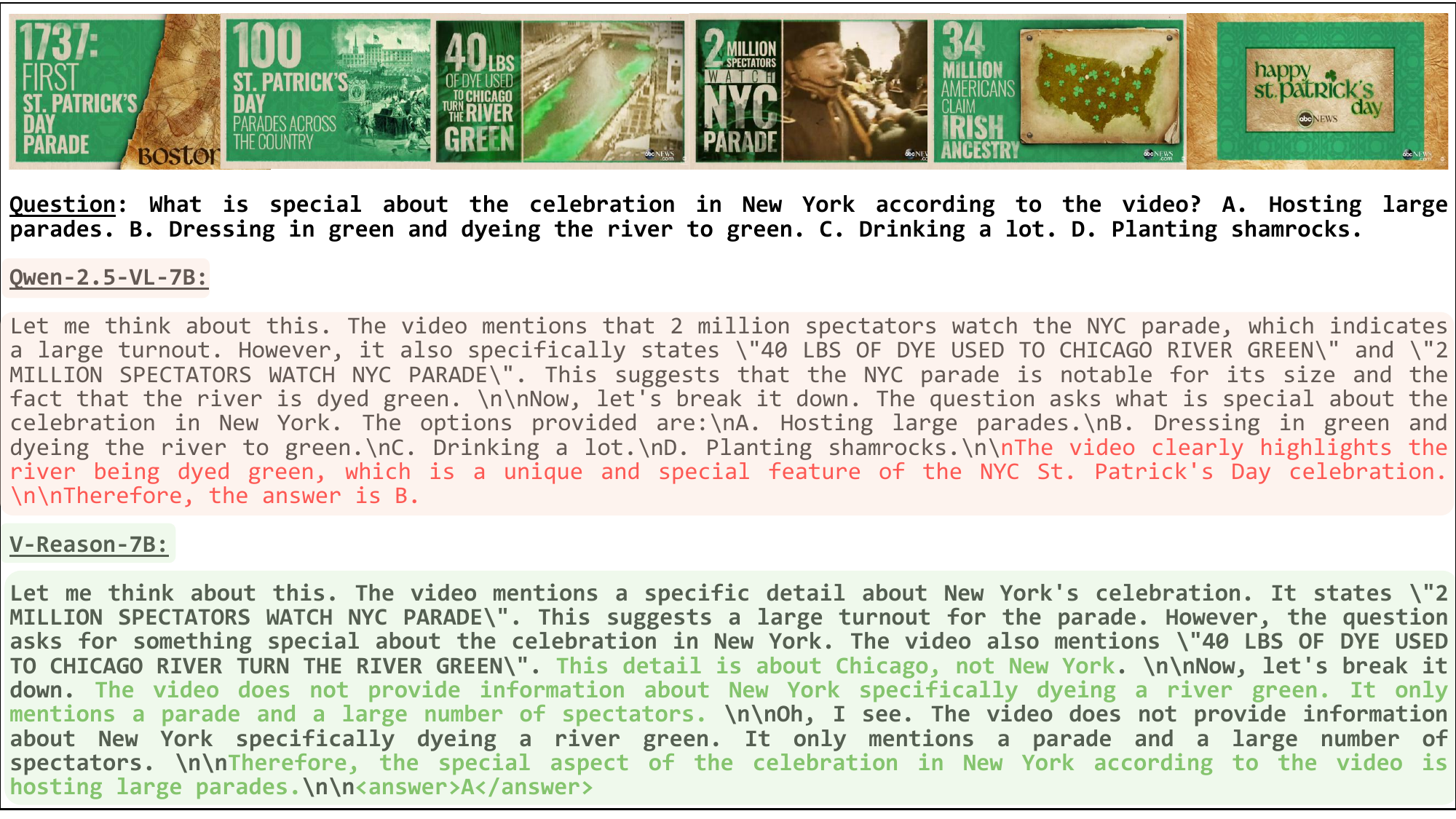}
\caption{Qualitative result: An example output and comparison of \ours{}-7B with the baseline Qwen-2.5-VL-7B.}
\label{fig:qual_3j}
\end{figure*}

\begin{figure*}[h]
\centering
\includegraphics[keepaspectratio, width=\textwidth]{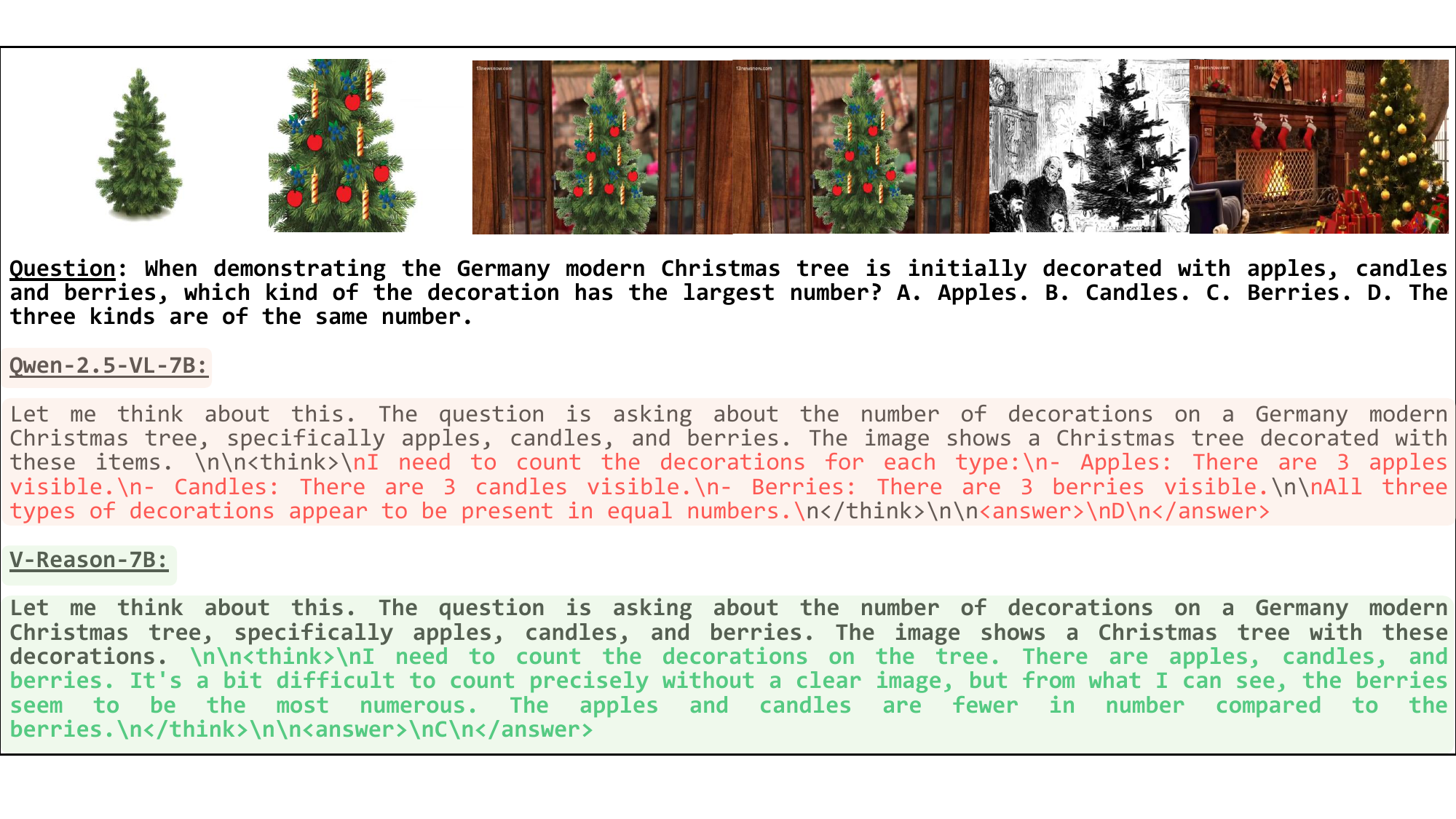}
\caption{Qualitative result: An example output and comparison of \ours{}-7B with the baseline Qwen-2.5-VL-7B.}
\label{fig:qual_3k}
\end{figure*}

\begin{figure*}[h]
\centering
\includegraphics[keepaspectratio, width=\textwidth]{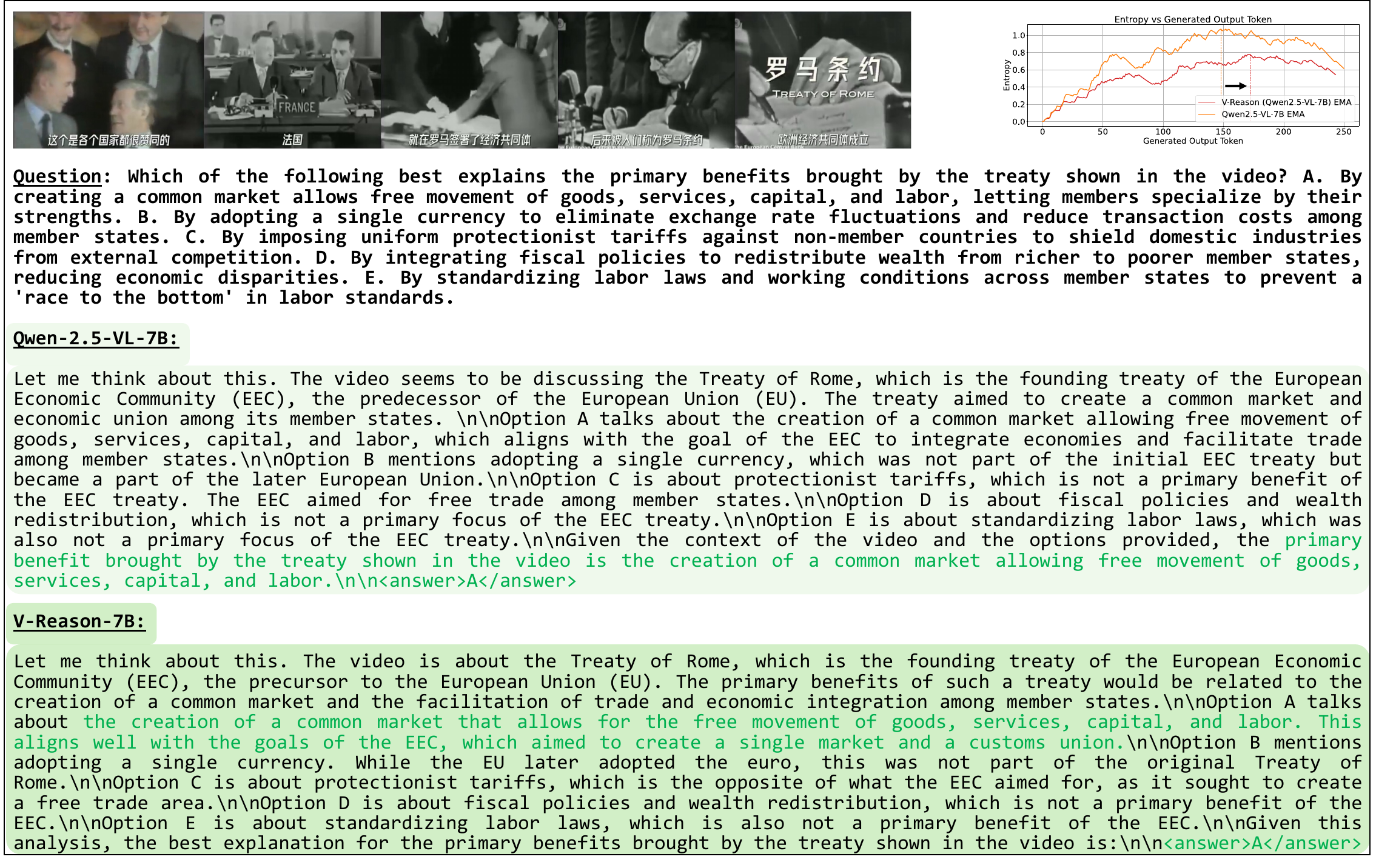}
\caption{Qualitative result: An example output and comparison with the baseline Qwen-2.5-VL-7B. Black arrow denotes the shift in the peak demonstrating longer exploration for \ours{}-7B.}
\label{fig:qual_3f}
\end{figure*}

\begin{figure*}[h]
\centering
\includegraphics[keepaspectratio, width=\textwidth]{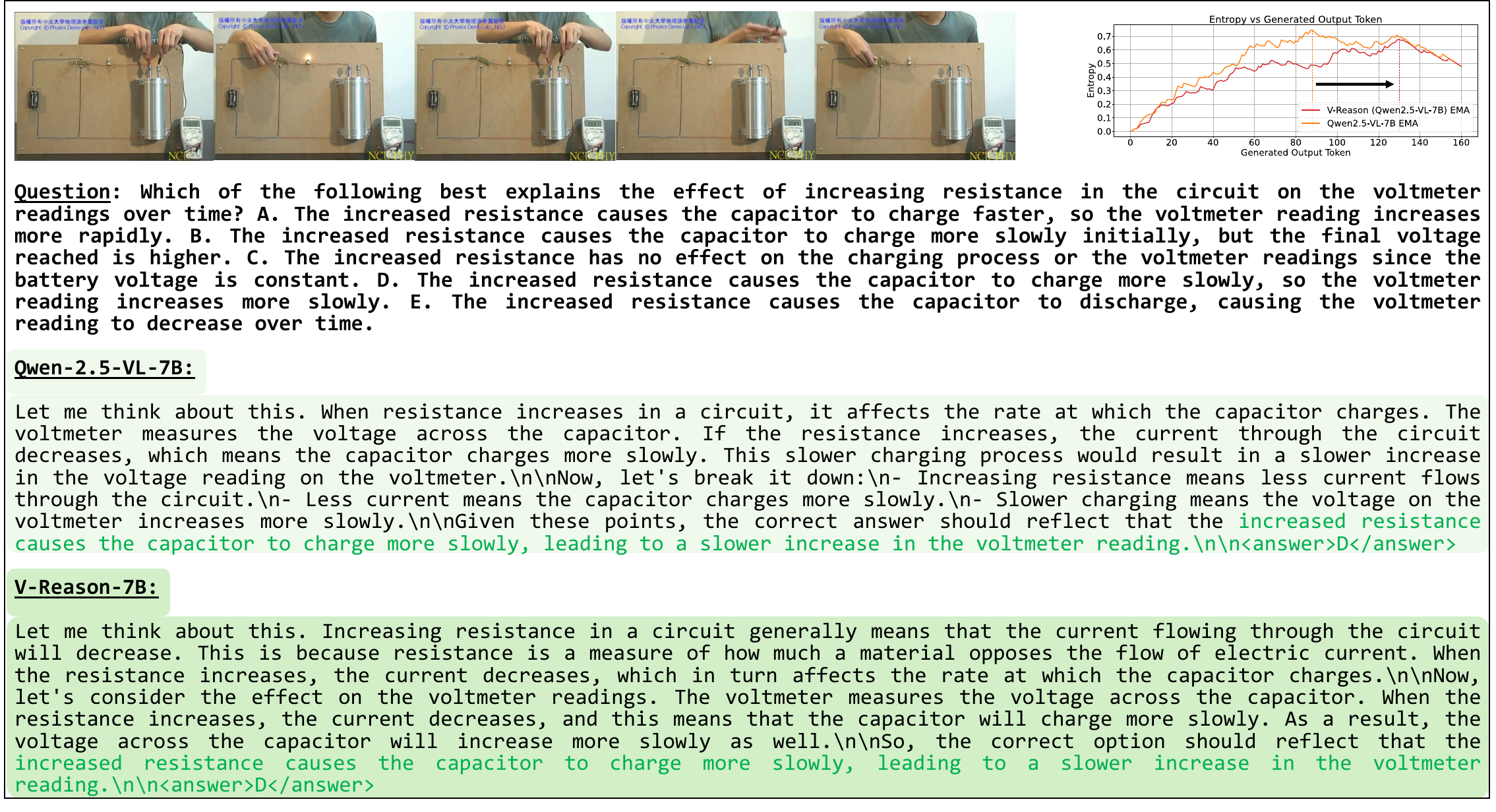}
\caption{Qualitative result: An example output and comparison with the baseline Qwen-2.5-VL-7B. Black arrow denotes the shift in the peak demonstrating longer exploration for \ours{}-7B.}
\label{fig:qual_3g}
\end{figure*}

\begin{figure*}[h]
\centering
\includegraphics[keepaspectratio, width=\textwidth]{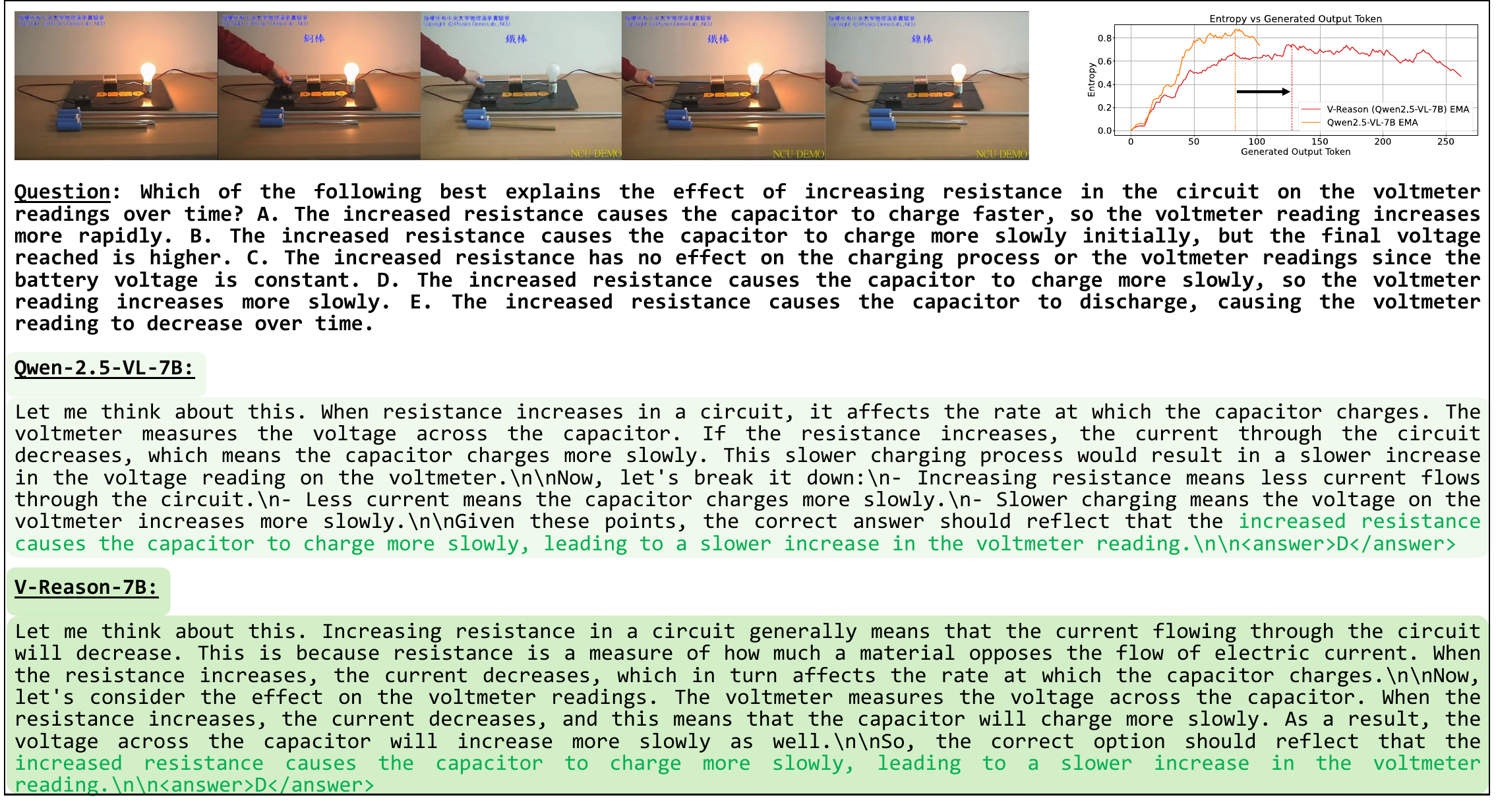}
\caption{Qualitative result: An example output and comparison with the baseline Qwen-2.5-VL-7B. Black arrow denotes the shift in the peak demonstrating longer exploration for \ours{}-7B.}
\label{fig:qual_3h}
\end{figure*}

\begin{figure*}[h]
\centering
\includegraphics[keepaspectratio, width=\textwidth]{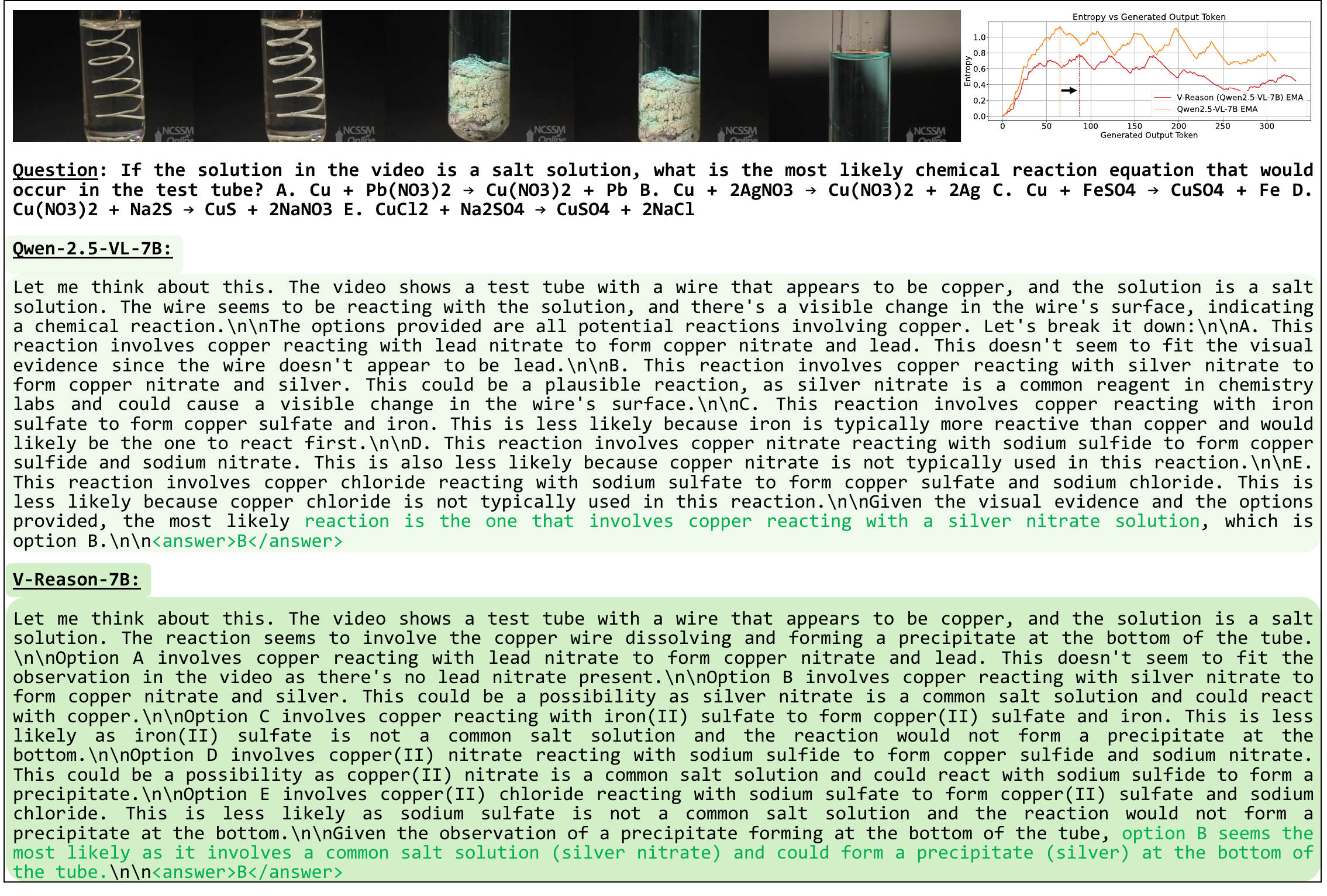}
\caption{Qualitative result: An example output and comparison with the baseline Qwen-2.5-VL-7B. Black arrow denotes the shift in the peak demonstrating longer exploration for \ours{}-7B.}
\label{fig:qual_3i}
\end{figure*}

\end{document}